\documentclass[11pt]{article}
\usepackage[margin=1in]{geometry}
\usepackage{dsfont}
\usepackage{booktabs, amsmath, amsthm, colortbl}
\newcommand{\R}{\mathbb{R}}
\newcommand{\smallpar}[1]{\medskip\noindent\textbf{#1.}\hspace{1ex}}

\newcommand{\C}{\mathcal{C}}
\newcommand{\X}{\mathcal{X}}

\newcommand{\T}{\mathsf{T}}
\newcommand{\leaves}{\mathsf{leaves}}
\newcommand{\pur}{\mathsf{pur}}
\newcommand{\cost}{\mathsf{cost}}
\newcommand{\DP}{\mathsf{DP}}
\newcommand{\MW}{\mathsf{MW}}
\newcommand{\calP}{\mathcal{P}}

\usepackage{color}
\usepackage{hyperref}
\usepackage[dvipsnames,svgnames,x11names,hyperref]{xcolor}
\definecolor{blueish}{HTML}{007F99}
\definecolor{purpleish}{HTML}{72177A}
\hypersetup{
	colorlinks,
	citecolor=purpleish,
	filecolor=red,
	linkcolor=blueish,
	urlcolor=Blue
}

\newcommand{\norm}[1]{\left\lVert#1\right\rVert}
\newtheorem{Theorem}{Theorem}
\newtheorem{Corollary}{Corollary}
\newtheorem{Lemma}{Lemma}
\usepackage[labelformat=simple]{subcaption}

\usepackage{natbib}

\usepackage[utf8]{inputenc} 
\usepackage[T1]{fontenc}    
\usepackage{hyperref}       
\usepackage{url}            
\usepackage{booktabs}       
\usepackage{amsfonts}       
\usepackage{nicefrac}       
\usepackage{microtype}      
\usepackage{graphicx}
\usepackage{bbm}
\usepackage{graphicx}
\usepackage{chemfig}

\begin{document}

\title{Unsupervised Embedding of Hierarchical Structure\\ in Euclidean Space\thanks{Jinyu Zhao and Yi Hao contributed equally.}}
	\author{
	}
	\date{\today}
	
	\author{Jinyu Zhao%
 		\thanks{%
 			{Dept. of Electrical \& Computer Engineering, University of California, San Diego (\url{jiz077@eng.ucsd.edu}).}
 			}
 		\and Yi Hao%
 		\thanks{%
 			{Dept. of Electrical \& Computer Engineering, University of California, San Diego (\url{yih179@eng.ucsd.edu}).}}
 		\and Cyrus Rashtchian%
 		\thanks{%
 			{Dept. of Computer Science \& Engineering, University of California, San Diego (\url{crashtchian@eng.ucsd.edu})}.}
 		    }
	
	\maketitle

\begin{abstract}
Deep embedding methods have influenced many areas of unsupervised learning. However, the best methods for learning hierarchical structure use non-Euclidean representations, whereas Euclidean geometry underlies the theory behind many hierarchical clustering algorithms. To bridge the gap between these two areas, we consider learning a non-linear embedding of data into Euclidean space as a way to improve the hierarchical clustering produced by agglomerative algorithms. To learn the embedding, we revisit using a variational autoencoder with a Gaussian mixture prior, and we show that rescaling the latent space embedding and then applying Ward's linkage-based algorithm leads to improved results for both dendrogram purity and the Moseley-Wang cost function. Finally, we complement our empirical results with a theoretical explanation of the success of this approach. We study a synthetic model of the embedded vectors and prove that Ward's method exactly recovers the planted hierarchical clustering  with high probability. 
\end{abstract}

\section{Introduction}

Hierarchical clustering aims to find an iterative grouping of increasingly large subsets of data, terminating when one cluster contains all of the data. This results in a tree structure, where each level determines a partition of the data. Organizing the data in such a way has many applications, including automated taxonomy construction and phylogeny reconstruction~\cite{balaban2019treecluster, friedman2001elements, manning2008introduction, shang2020nettaxo, zhang2018taxogen}. Motivated by these applications, we address the simultaneous goals of recovering an underlying clustering and building a hierarchical tree of the entire dataset. We imagine that we have access to many samples from each individual cluster (e.g., images of plants and animals), while we lack labels or any direct information about the hierarchical relationships. In this scenario, our objective is to correctly cluster the samples in each group and also build a dendrogram within and on top of the clusters that matches the natural supergroups.

Despite significant effort, hierarchical clustering remains a challenging task, theoretically and empirically. Many objectives are NP-Hard~\citep{charikar2017approximate,charikar2018hierarchical,  dasgupta2002performance, dasgupta2016cost,heller2005bayesian, MoseleyWang}, and many theoretical algorithms may not be practical because they require solving a hard subproblem (e.g., sparsest cut) at each step~\citep{ackermann2014analysis, ahmadian2019bisect, charikar2017approximate, charikar2018hierarchical,  dasgupta2016cost, lin2010general}. The most popular hierarchical clustering algorithms are linkage-based methods~\cite{king1967step, lance1967general, sneath1973numerical, ward1963hierarchical}. These algorithms first partition the dataset into singleton clusters. Then, at each step, the pair of most similar clusters is merged. There are many variants depending on how the cluster similarity is computed. While these heuristics are widely used, they do not always work well in practice. At present, there is only a superficial understanding of theoretical guarantees, and the central guiding principle is that linkage-based methods recover the underlying hierarchy as long as the dataset satisfies certain strong separability assumptions~\citep{abboud2019subquadratic, balcan2019learning, Cohen-Addad, emamjomeh2018adaptive, ghoshdastidar2019foundations, grosswendt2019analysis,kpotufepruning,reynolds2006clustering, roux2018comparative, sharma2019comparative,szekely2005hierarchical, tokuda2020revisiting}. 

We consider embedding the data in Euclidean space and then clustering with a linkage-based method. Prior work has shown that a variational autoencoder (VAE) with Gaussian mixture model (GMM) prior leads to separated clusters via the latent space embedding~\cite{dilokthanakul2016deep,nalisnick2016approximate,jiang2017variational,uugur2020variational, xie2016unsupervised}. 
We focus on one of the best models, Variational Deep Embedding~(VaDE)~\cite{jiang2017variational}, in the context of hierarchical clustering. Surprisingly, the VaDE embedding seems to capture hierarchical structure. Clusters that contain semantically similar data end up closer in Euclidean distance, and this pairwise distance structure occurs at multiple scales. This phenomenon motivates our empirical and theoretical investigation of enhancing hierarchical clustering quality with unsupervised embeddings into Euclidean space.

Our experiments demonstrate that applying Ward's method after embedding with VaDE produces state-of-the-art results for two hierarchical clustering objectives: (i) dendrogram purity for classification datasets~\cite{heller2005bayesian} and (ii) Moseley-Wang's cost function~\cite{MoseleyWang}, which is a maximization variant of Dasgupta's cost~\cite{dasgupta2016cost}. Both measures reward clusterings for merging more similar data points before merging different ones. For brevity, we restrict our attention to these two measures, while our qualitative analysis suggests that the VaDE embedding would also improve other objectives.

As another contribution, we propose an extension of the VaDE embedding that rescales the means of the learned clusters (attempting to increase cluster separation). On many real datasets, this rescaling improves both clustering objectives, which is consistent with the general principle that linkage-based methods perform better with more separated clusters. The baselines for our embedding evaluation involve both a standard VAE with isotropic Gaussian prior~\cite{kingma2014auto, rezende2014stochastic} and principle component analysis (PCA). In general, PCA provides a worse embedding than VAE/VaDE, which is expected because it cannot learn a non-linear transformation. The VaDE embedding leads to better hierarchical clustering results than VAE, indicating that the GMM prior offers more flexibility in the latent space. 

Our focus on Euclidean geometry lends itself to a synthetic distribution for hierarchical data. It will serve both as a challenging evaluation distribution and as a model of the VaDE embedding. We emphasize that we implicitly represent the hierarchy with pairwise Euclidean distances, bypassing the assumption that the algorithm has access to actual similarities. Extending our model further, we demonstrate a shifted version, where using a non-linear embedding  of the data improves the clustering quality.
Figure~\ref{synthetic_data-1} depicts the original and shifted synthetic data with 8 clusters forming a 3-level hierarchy. The 3D plot of the original data in Figure~\ref{fig:3d-1} has ground truth pairwise distances in Figure~\ref{fig:truth-1}. In Figure~\ref{fig:shift-orig-1}, we see the effect of {\em shifting} four of the means by a cyclic rotation, e.g., $(4,2,1,0,0,0)\mapsto (0,0,0,4,2,1)$. This non-linear transformation increases the distances between pairs of clusters while preserving hierarchical structure. PCA in Figure~\ref{fig:shift-pca-1} can only distinguish between two levels of the hierarchy, whereas VaDE in Figure~\ref{fig:shift-vade-1} leads to  concentrated clusters and identifies multiple levels. 

To improve the theoretical understanding of Ward's method, we prove that it exactly recovers both the underlying clustering and the planted hierarchy when the cluster means are hierarchically separated. The proof strategy involves bounding the distance between points sampled from spherical Gaussians and extending recent results on Ward's method for separated data~\cite{grosswendt2019analysis, grosswendt_2020}. We posit that the rescaled VaDE embedding produces vectors that resemble our planted hierarchical model, and this enables Ward's method to recover the clusters and the hierarchy after the embedding.

\subsection{Related Work}
For flat clustering objectives (e.g., $k$-means or clustering accuracy), there is a lot of research on using embeddings to obtain better results by increasing cluster separation~\cite{aljalbout2018clustering, dilokthanakul2016deep, fard2018deep, figueroa2017simple, guo2017deep,jiang2017variational, li2018discriminatively,li2018learning, min2018survey, tzoreff2018deep, xie2016unsupervised, yang2017towards}. However, we are unaware of prior work on learning unsupervised embeddings to improve hierarchical clustering. Studies on hierarchical representation learning instead focus on metric learning or maximizing data likelihood, without evaluating downstream unsupervised hierarchical clustering algorithms~\cite{ganea2018hyperbolic,
goldberger2005hierarchical, goyal2017nonparametric, heller2005bayesian, klushyn2019learning, nalisnick2017stick, shin2019hierarchically, sonthalia2020tree, teh2008bayesian,tomczak2018vae, vasconcelos1999learning}. Our aim is to embed the data into Euclidean space 
so that linkage-based algorithms produce high-quality clusterings.

Researchers have recently identified the benefits of hyperbolic space for representing hierarchical structure~\citep{chami2020trees, de2018representation, gu2018learning, linial1995geometry, mathieu2019continuous, monath2019gradient, nickel2017poincare, tifrea2018poincar}. Their motivation is that hyperbolic space can better represent tree-like metric spaces due to its negative curvature~\citep{sarkar2011low}. We offer an alternative perspective. If we commit to using a linkage-based method to recover the hierarchy, then it suffices to study Euclidean geometry. Indeed, we do not need to approximate all pairwise distances --- we only need to ensure that the clustering method produces a good approximation to the true hierarchical structure.  Our approach facilitates the use of off-the-shelf algorithms designed for Euclidean geometry, whereas other methods, such as hyperbolic embeddings, require new tools for downstream tasks. 

Regarding our synthetic model, prior work has studied analogous models that sample pairwise {\em similarities} from separated Gaussians~\cite{balakrishnan2011noise, Cohen-Addad, ghoshdastidar2019foundations}. Due to this key difference, previous techniques for similarity-based planted hierarchies do not apply to our data model. We extend the prior results to analyze Ward's method directly in Euclidean space (consistent with the embeddings).


\begin{figure*}[t!]
    \centering
    \begin{subfigure}[t]{0.44\textwidth}
        \centering
        \includegraphics[height=1.5in]{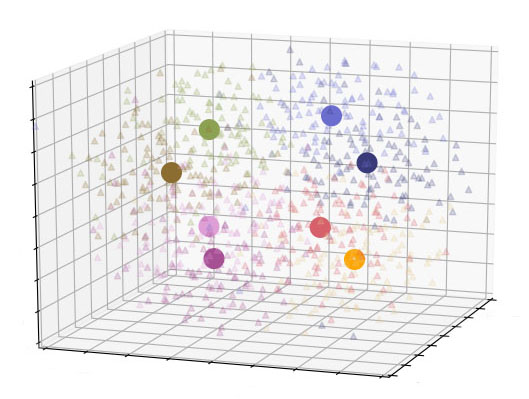}
        \caption{3D plot of three-level hierarchy (8 clusters)}
        \label{fig:3d-1}
    \end{subfigure}%
    ~ 
    \begin{subfigure}[t]{0.44\textwidth}
        \centering
        \includegraphics[height=1.5in]{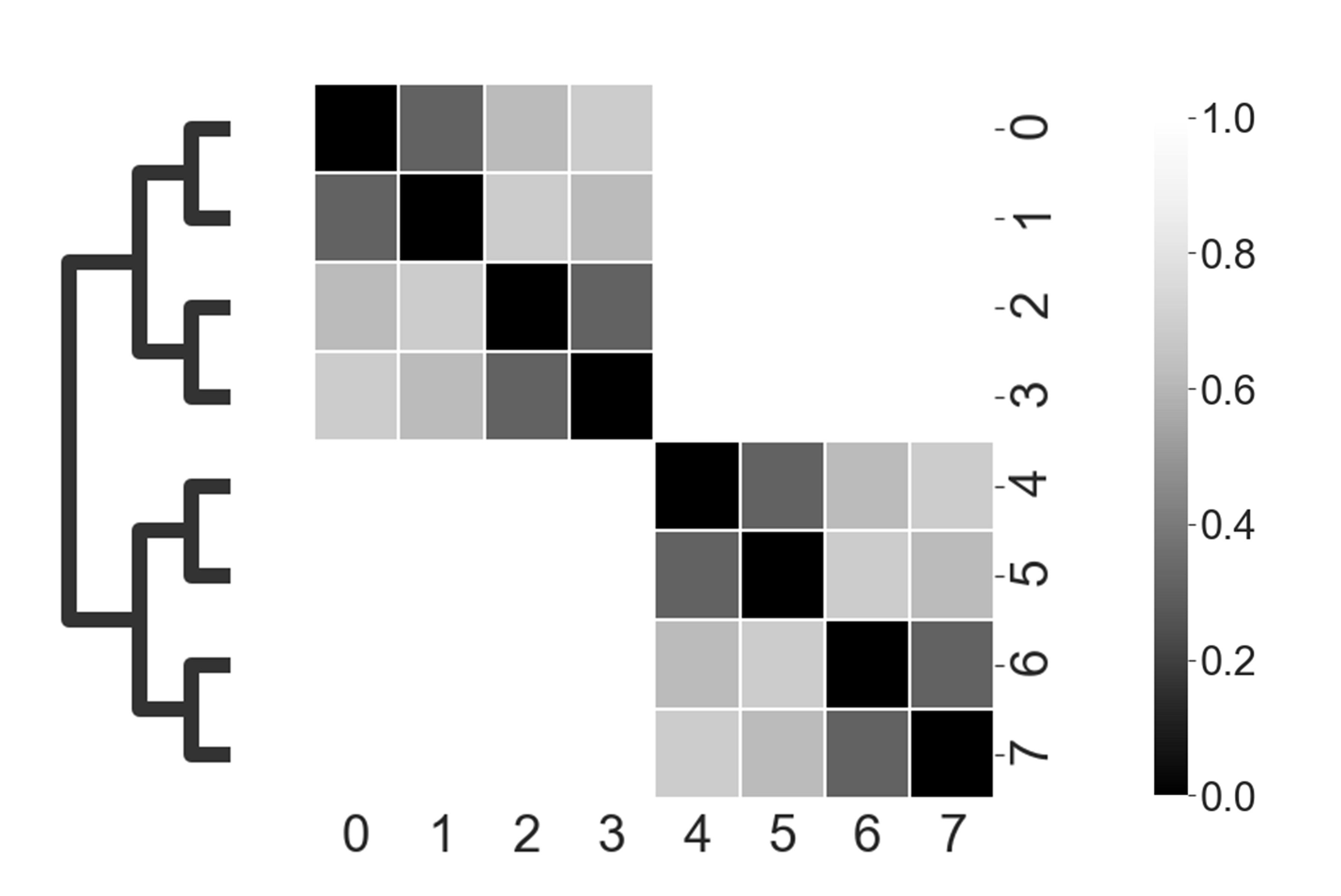}
        \caption{Ground truth dendrogram \& distance matrix\vspace{1ex}}
        \label{fig:truth-1}
    \end{subfigure}
    \begin{subfigure}[t]{0.3\textwidth}
        \centering
        \includegraphics[height=1.3in]{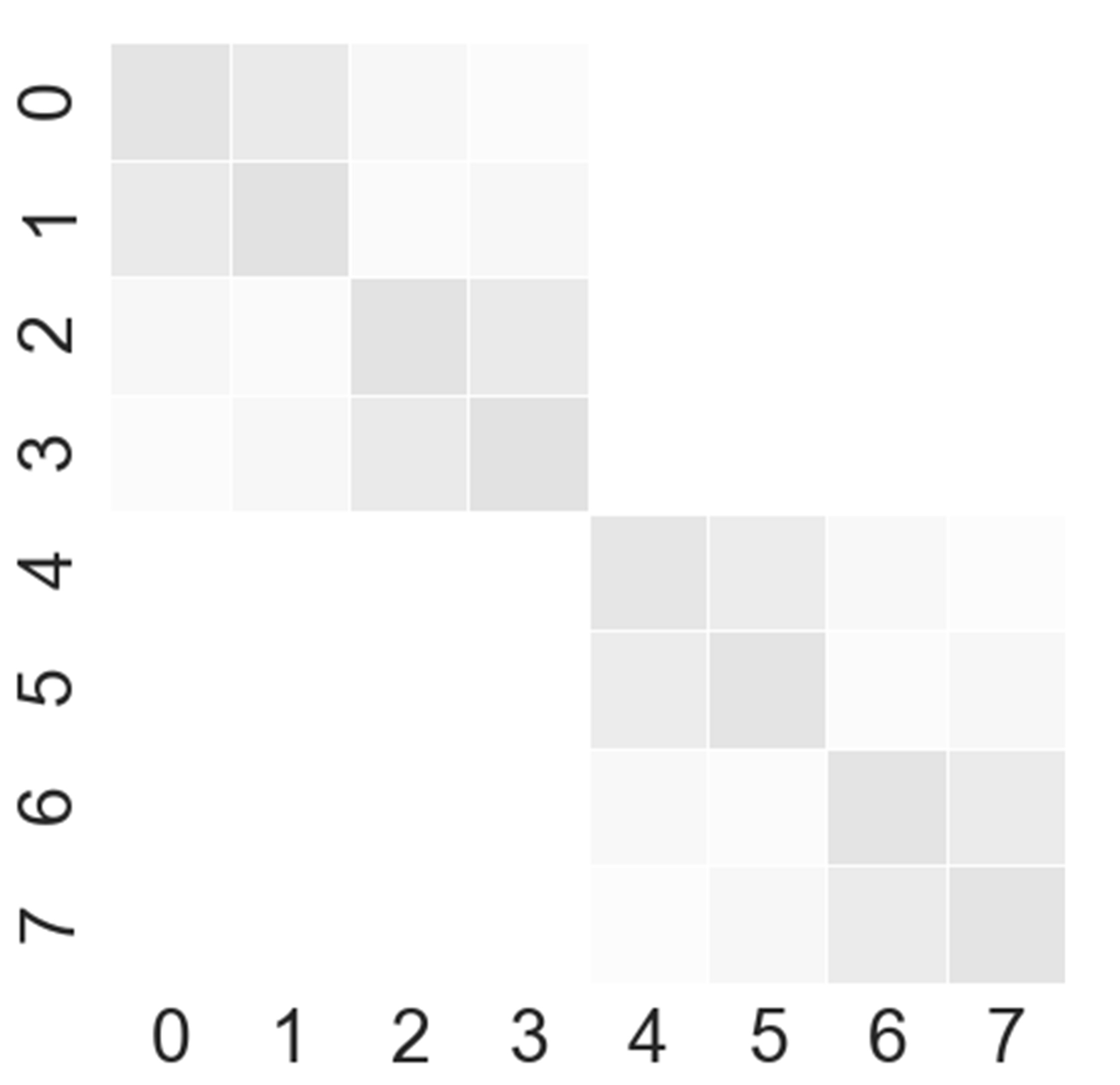}
        \caption{Dataset with shifted means}
        \label{fig:shift-orig-1}
    \end{subfigure}%
    ~ 
    \begin{subfigure}[t]{0.3\textwidth}
        \centering
        \includegraphics[height=1.3in]{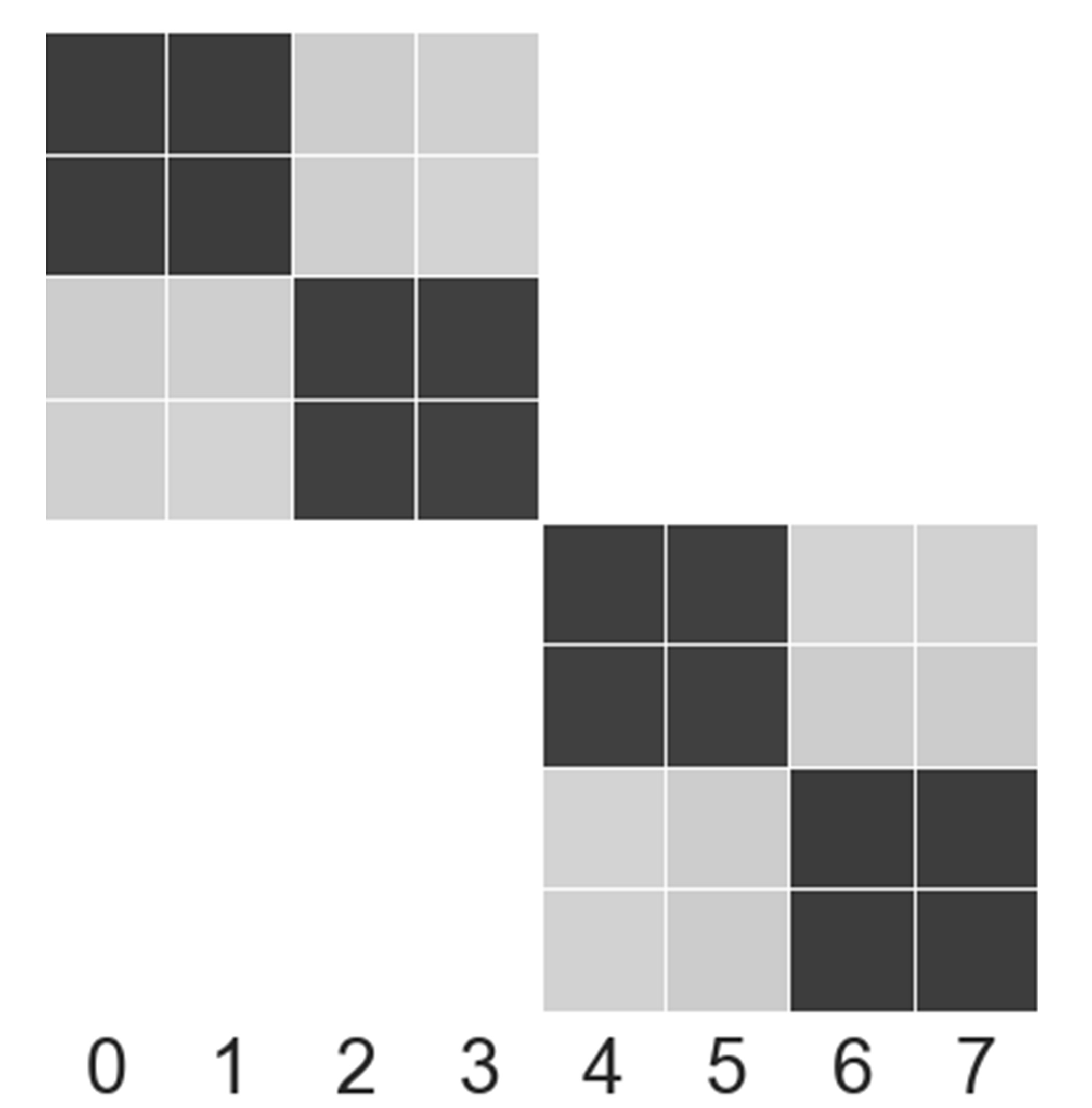}
        \caption{PCA on shifted dataset}
        \label{fig:shift-pca-1}
    \end{subfigure}
    ~
    \begin{subfigure}[t]{0.3\textwidth}
        \centering
        \includegraphics[height=1.3in]{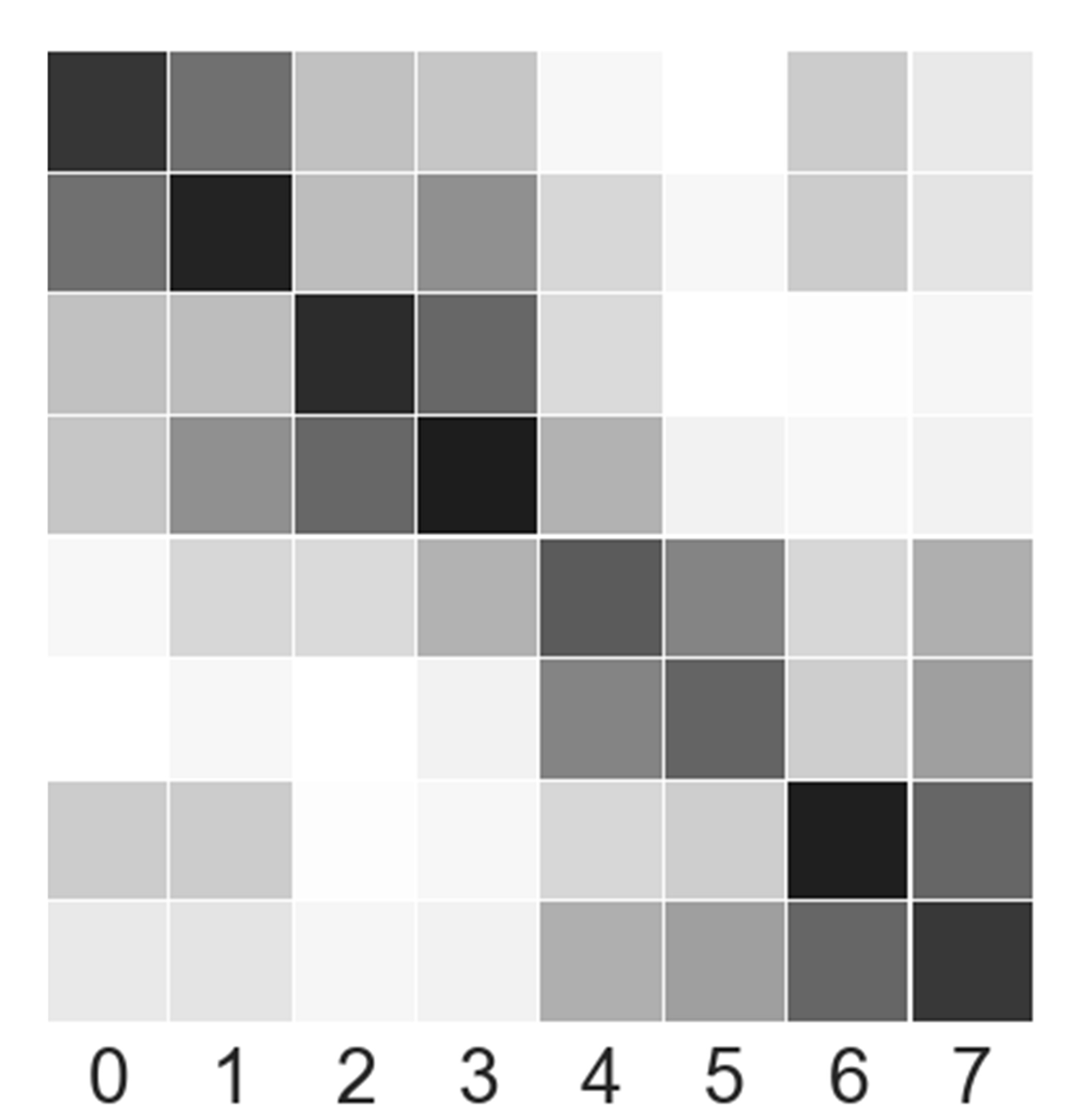}
        \caption{VaDE on shifted dataset}
        \label{fig:shift-vade-1}
    \end{subfigure}
    \caption{{\bf Top row:} Synthetic data demonstrating a three-level hierarchy encoded via the pairwise Euclidean distances between eight clusters. {\bf Bottom row:} Effect of shifting half of the means by cyclically rotating the vectors, e.g., $(4,2,1,0,0,0)\mapsto (0,0,0,4,2,1)$. Comparing PCA and VaDE shows that a non-linear embedding can increase cluster separation and roughly preserve the hierarchy.}
\label{synthetic_data-1}
\end{figure*}


\section{Clustering Models and Objectives}
\label{sec:prelims}

\smallpar{Planted Hierarchies with Noise} 
We introduce a simple heuristic of constructing a mixture of Gaussians that also encodes hierarchical structure via Euclidean distances. In particular, the resulting new construction (i) approximately models the latent space of VaDE, and (ii) allows us to test the effectiveness of VaDE in increasing the separation among clusters.

We shall refer to our model as {\bf BTGM}, Binary Tree Gaussian Mixture, formed by a uniform mixture of spherical Gaussians with shared variance and different means. The model is determined by four parameters: height $h$, margin $m$, and dimension $d$, and expansion ratio $\alpha$, a constant satisfying $\alpha > 1$. 
Specifically, in a BTGM with height $h$ and margin $m$, we construct $2^h$ spherical Gaussian distributions with  unit variance $I_d$ and mean vectors $\mu_{i}\in \mathbb R^d$, 
whose $j^{th}$ coordinate $\mu_i[j]$ is specified by the following equation.
\begin{equation}\label{eq:btgm}
    \mu_i[j] = 
    \begin{cases}
        (-1)^{p_{i,j}} \cdot m \cdot \alpha^{h-j}, 
        \textrm{ for }\\ 
        \hspace{2em}p_{i,j}\equiv i \!\!\!\!\pmod{2^{h - i}} \textrm{  and  } j \leq h\\  0, \ \textrm{if } h < j \leq d.
    \end{cases}
\end{equation}\par
\vspace{-0.6em}
At a high level, we constructed a complete binary tree of height $h$, with each leaf corresponding to a spherical Gaussian distribution. 
A crucial structural property is that the distance between any two means, $\mu_i$ and $\mu_j$, is positively related to their distance in the binary tree. More precisely, 
$$\norm{\mu_i - \mu_j}_2 \geq  m \cdot \alpha^ {\mathsf{ht}(\mathsf{LCA}(\mu_i, \mu_j))},$$
where $\mathsf{ht}(v)$ denotes the height of node $v$ in the tree, and $\mathsf{LCA}$ is the least common ancestor of two nodes. Figure~\ref{synthetic_data-1} depicts an instance of the BTGM model with parameters $m = 4$, $h = 3$, $\alpha=2$, and $d = 100$. Figure~\ref{fig:3d-1} presents the 3D tSNE visualization of the data, and Figure~\ref{fig:truth-1} shows the pairwise distance matrix (normalized to $[0,1]$) of the clusters and the ground-truth hierarchy. We also define a non-linear shifting operation on the BTGM, which increases the difficultly of recovering the structure and allows more flexibility in experiments. Specifically, we shift the nonzero mean-vector entries of the first $\frac{k}{2}$ BTGM components to arbitrary new locations, 
e.g., we may shift $\mu_1$ from $(4,2,1,0,0,0)$ to $ (0,0,0,4,2,1)$ via rotation in a BTGM with $k = 8, h = 3$, $\alpha=2$, and $m = 2$. After the operation, the model's hierarchical structure is less clear from the mean distances in the $h$-dimensional Euclidean subspace. Other non-linear  operations can be applied to model different data types.

\smallpar{Data Separability}
Our synthetic model and our theoretical results assume that the dataset satisfies certain separability requirements. Such properties do not hold for real datasets {\em before} applying the embedding. Fortunately, the assumptions are not about the original data, but about the transformed dataset. We observe that VaDE increases the separation between different clusters, and we also see that it preserves some of the hierarchical information as well. Therefore, in the rest of the paper, we consider separation assumptions that are justified by the fact that we apply a learned embedding to the dataset before running any clustering algorithm. For more details, see Figure \ref{fig:separation} in Appendix~\ref{app:more-experiments}, which shows the separation on MNIST.


\subsection{Hierarchical Clustering Objectives} 
Index the dataset as $\X = \{x_1,\ldots,x_n\}$. Let $\leaves(\T[i \vee j]) \subseteq \X$ denote the set of leaves of the subtree in $\T$ rooted at $\mathsf{LCA}(x_i, x_j)$, the least common ancestor of $x_i$ and $x_j$. Given pairwise similarities $w_{ij}$ between $x_i$ and $x_j$, {\bf Dasgupta's cost}~\cite{dasgupta2016cost} minimizes 
$
    \cost(\T) = 
    \sum_{i,j \in [n]}
    w_{ij} \cdot |\leaves(\T[i \vee j])|.
$
{\bf Moseley-Wang's objective}~\cite{MoseleyWang} is a dual formulation, based on maximizing
\begin{equation}\label{eq:mw}
    \MW(\T) = 
    n \sum_{i,j \in [n]} w_{ij} - \cost(\T).
\end{equation}
{\bf Dendrogram purity}~\cite{heller2005bayesian} uses ground truth clusters $\C = \{C^*_1,\ldots, C^*_k\}$. Let $\calP^*$ be the set of pairs in the same cluster: $\calP^* = 
\{(x^i,x^j) \mid \C(x^i) = \C(x^j)\}$. Then, 
\begin{equation}\label{eq:dp}
  \DP(\T)\! = \frac{1}{|\calP^*|}
  \sum_{\ell = 1}^k 
  \sum_{x^i,x^j \in C^*_\ell}\!\!
  \pur(\leaves(\T[i \vee j]), C^*_\ell),
  \end{equation}
where
$
\pur(A,B) = |A \cap B|/|A|.
$

\smallpar{Ward's Method} 
Let $C_1,\ldots, C_k$ be a $k$-clustering. Let $\mu(C) = \frac{1}{|C|}\sum_{x\in C} x$ denote the mean of cluster~$C$. Ward's method~\cite{ward1963hierarchical, grosswendt2019analysis} merges two clusters to minimize the increase in sum of squared distances: \begin{align*}
\min_{i,j} \sum_{x \in C_i \cup C_j} \|x - \mu(C_i \cup C_j)\|_2^2 - \sum_{x \in C_i} \|x - \mu(C_i)\|_2^2 - \sum_{x \in C_j} \|x - \mu(C_j)\|_2^2.\end{align*}

\section{Embedding with VaDE}
\label{sec:approach}

VaDE~\cite{jiang2017variational} considers the following generative model specified by $p(x,z,c) = p(x\mid z)p(z\mid c)p(c)$. 
We use $\mathsf{Cat}(.)$ to denote a categorical distribution, $\mathcal{N}(.)$ is a normal distribution, and  $k$ is the predefined number of clusters. In the following, $\mu_c$ and $\sigma^2_c$ are parameters for the latent distribution, and $\hat \mu_z$ and $\hat \sigma^2_z$ are outputs of $f(z ;\theta)$, parametrized by a neural network.
And the generative process for the observed data is
\begin{equation}
    c \sim \mathsf{Cat}(1 / k),\  z \sim \mathcal{N}(\mu_c, \sigma^2_c \cdot I),\  x \sim \mathcal{N}(\hat \mu_z, \hat \sigma^2_z \cdot I).
\end{equation}
 The VaDE maximizes the likelihood of the data, which is bounded by the evidence lower bound, 
\begin{equation}\label{eq:elbo}
\begin{split}
L_{\text{\tiny ELBO}}
&= E_{q(z, c |x)}[\log p(x \!\mid\! c)] 
\\
&\hspace{5em}- D_{\text{KL}}(q(z,c\! \mid\! x)\ \|\ p(z, c)),
\end{split}
\end{equation}
where $q(z,c \mid x)$ is a variational posterior to approximate the true posterior $p(z,c \mid x)$. We use a neural network $g(x;\phi)$ to parametrize the model $q(z \mid x)$. We refer to the original paper for the derivation of Eq. (\ref{eq:elbo}) and for the connection of $L_{\text{ELBO}}$ to the gradient formulation of the objective~\cite{jiang2017variational}.

\subsection{Improving separation by rescaling means}
The VaDE landscape is a mixture of multiple Gaussian distributions.
As we will see in Section 4, the separation of Gaussian distributions is crucial for Ward's method to recover the underlying clusters. Since Ward's method is a local heuristic, data points that lie in the middle area of two Gaussians are likely to form a larger cluster before merging into one of the true clusters. 
This ``problematic'' cluster will further degrade the purity of the dendrogram. However, the objective that VaDE optimizes does not encourage a larger separation. As a result, we apply a transformation to the VaDE latent space to enlarge the separation between each Gaussians.  Let $x'_i$ be the embedded value of the data point $x_i$, and define 
$$
\mathcal{C}(x'_i) = \mathrm{argmax}_j[p(x'_i \mid \mu_{c = j},\ \sigma^2_{c = j}\cdot I)]
$$ 
to be the cluster label assigned by the learned GMM model. Then, we compute the transformation
\begin{equation*}
    x''_i = x'_i + s \cdot \mu_{i}
    \text{ with } \mu_i=\mu_{c=\mathcal{C}(x_i)},
\end{equation*}
where $s$ is a positive rescaling factor. The transformation has the following properties:
(i) for points assigned to the same Gaussian by the VaDE method, the transformation preserves their pairwise distances;
(ii) if $x_i$'s
assigned to the same 
$\mu$ by $\mathcal C$
are i.i.d.~random variables with expectation $\mu$ (which coincides with the intuition behind VaDE), then the transformation preserves the ratio between the distances of the expected point locations associated with different mean values. 
%
%

Appendix~\ref{app:more-experiments} illustrates how to choose $s$ in detail. To improve both DP and MW, the rule of thumb is to choose a scaling factor~$s$ large enough. We set $s = 3$ for all datasets.  In summary, we propose the following hierarchical clustering procedure:


\begin{enumerate}
    \item Train VaDE on the unlabeled dataset $\{x_1,\ldots, x_n\}$.
    \item Embed $x_i$ using the value mapping $x_i' = \hat \mu_{\hat z}$ where $(\hat \mu_{\hat z}, \log \hat \sigma_{\hat z}^2) = f(\hat z; \ \theta)$ and $\hat z = g(x_i\ ;\ \phi)$.
    \item Choose a rescaling factor $s$ and apply the rescaling transformation on each $x'_i$ to get $x''_i$.
    \item Run Ward's method on $\{x''_i\}_{i=1}^n$ to produce a hierarchical clustering.
\end{enumerate}


\begin{figure*}[t!]
    \centering
    \begin{subfigure}[t]{0.22\textwidth}
        \centering
        \includegraphics[height=0.95in]{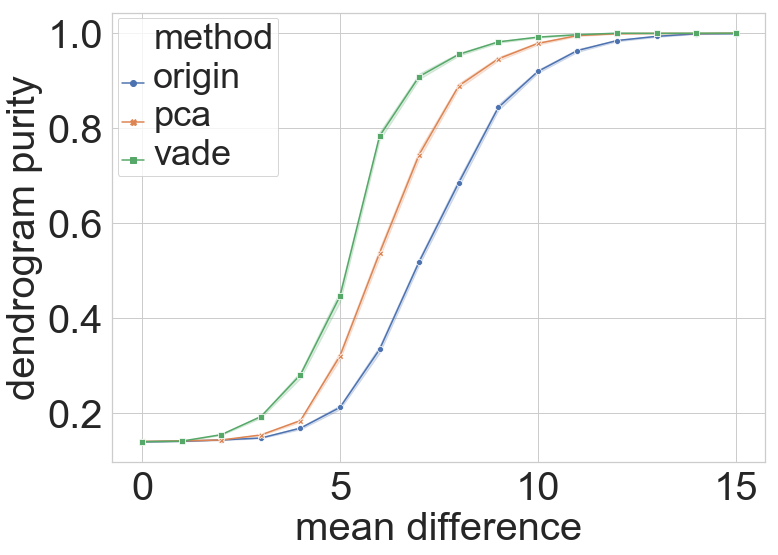}
        \caption{Dendrogram Purity}
        \label{fig:shift-orig}
    \end{subfigure}%
    ~ 
    \begin{subfigure}[t]{0.22\textwidth}
        \centering
        \includegraphics[height=0.95in]{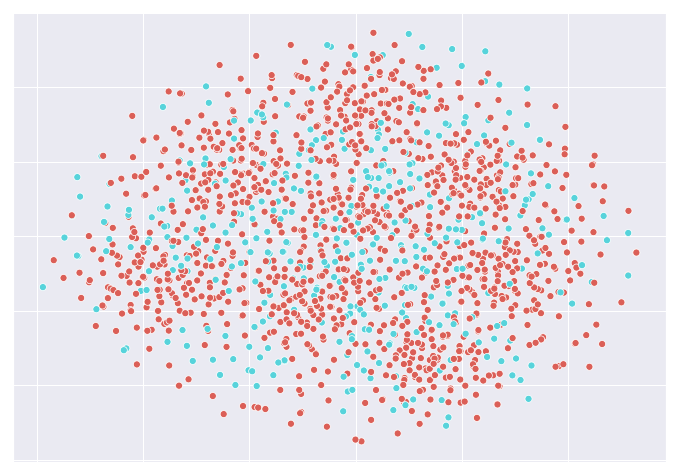}
        \caption{Original + Ward}
        \label{fig:spherical-original}
    \end{subfigure}
    ~
    \begin{subfigure}[t]{0.22\textwidth}
        \centering
        \includegraphics[height=0.95in]{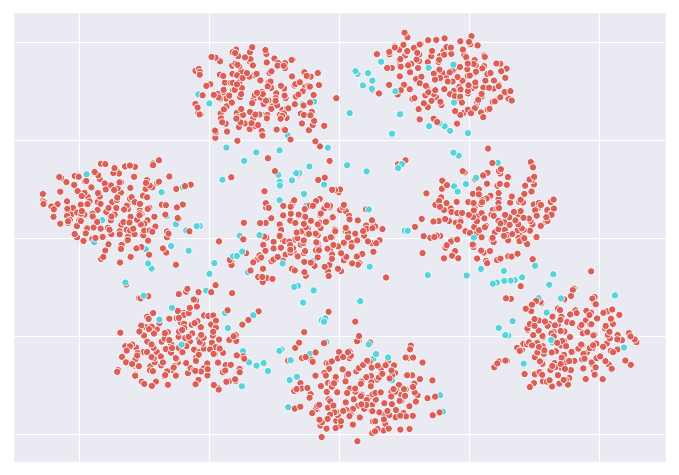}
        \caption{PCA + Ward}
        \label{fig:spherical-pca}
    \end{subfigure}
    ~
    \begin{subfigure}[t]{0.22\textwidth}
        \centering
        \includegraphics[height=0.95in]{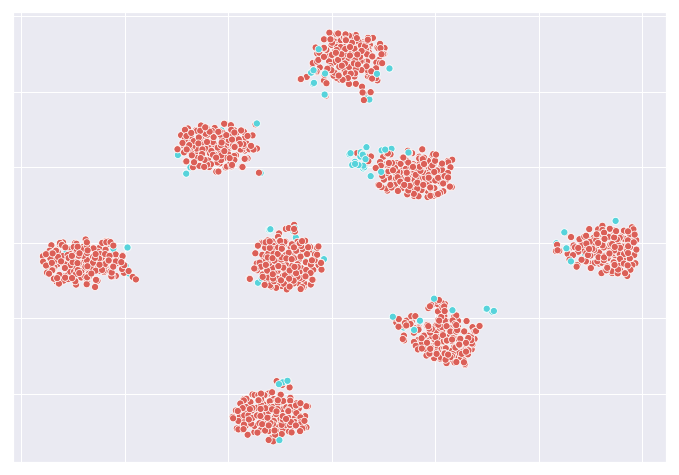}
        \caption{VaDE + Ward}
        \label{fig:spherical-vade}
    \end{subfigure}
    \caption{(a) DP of Ward's method on original, PCA, and VaDE space, varying mean separation. (b)(c)(d) show 8 clusters recovered by running Ward's method on original, PCA and VaDE latent space (Ward's method terminates when there are 8 clusters). Data drawn from 8 spherical Gaussians in $\R^{100}$ with pairwise mean separation 8. Point clouds generated by tSNE; Red colors indicate correct cluster labels assigned by Ward's method, and blue indicates wrong cluster.}
\label{fig:ward_experiment}
\end{figure*}

\section{Theoretical Analysis}
\label{sec:theory}






We provide theoretical guarantees for Ward's method on separated data, providing justification for using VaDE (proofs in Appendix~\ref{theorem_proof}). Here, we focus on the mixture model of \emph{spherical Gaussians}.  We are given a sample of size $n$, say $X_1, \ldots, X_n$,  where each point $X_i\in \mathbb R^d$ is independently drawn  from one of $k$ Gaussian components according to mixing weights $w_1,\ldots, w_k$, with the $j$-th component having mean $\mu_j\in \mathbb R^d$ and 
covariance matrix $\Sigma_j\in \mathbb R^{d\times d}$.  Specifically, for any component index $j$, 
we assume that the covariance $\Sigma_j$ satisfies $\Sigma_j = \sigma_j^2 \cdot I_d$, 
where $\sigma_j>0$ and $I_d\in \mathbb R^{d\times d}$ is the identity matrix. 
Under this assumption, the model is fully determined by the 
set of parameters $\{(w_i,\mu_i,\sigma_i): \forall i\in [k]\}$. 
The goal is to cluster the sample points by their associated components. 

The theorem below presents sufficient conditions for Ward's method to recover the 
underlying $k$-component clustering when points are drawn from a $k$-mixture of spherical Gaussians. To recover a $k$-clustering, we terminate when there are $k$ clusters.

\begin{Theorem}~\label{WardMethod}
Let $c, c_0>0$ be absolute constants.
Suppose we have $n$ samples from a $k$-mixture of spherical Gaussians
with parameters $\{(w_i,\mu_i,\sigma_i): w_i>0, \forall i\in [k]\}$ satisfying
\[
\forall i\not= j\in [k],
\norm{\mu_i-\mu_j}_2\ge 
c\sqrt{\nu}(\sigma_i+\sigma_j)(\sqrt d + \sqrt{\log n}),
\]
where $\nu:=\max_{\ell \not= t}w_\ell/w_t$, 
and suppose $n\ge \frac{c_0}{w_{\text{\tiny min}}}\log k $, where 
$w_{\text{\tiny min}}\!:=\!\min_t w_t$. Ward's method will recover the underlying $k$-clustering with probability at least $1-2/k$.
\end{Theorem} 
As a recap, Ward's method starts with singleton clusters, and repeatedly 
selects a pair of clusters and merges them into a new cluster,
where the selection is based on the optimal value of the error sum of squares. 
In particular, \emph{recovering the underlying $k$-component clustering} means that 
every cluster contains only points from the same Gaussian component
when the algorithm produces a $k$-clustering. 

\subsection{Intuition and Optimality} 
On a high level,
the mean separation conditions in  Theorem~\ref{WardMethod} guarantee that
with high probability, the radiuses of the clusters will be larger than the inter-cluster distances, and the sample lower bound ensures 
that we observe at least constant many sample points from every Gaussian, with high probability.
In Appendix~\ref{theorem_optimality}, 
we illustrate the hardness of improving the theorem without pre-processing the data, e.g., 
applying an embedding method to the sample points, such as VaDE or PCA. 
In particular, we show that
(i) without further complicating the algorithm, 
the separation conditions in Theorem~\ref{WardMethod} 
have the right dependence on $d$ and $n$; (ii)
for exact recovery, the sample-size lower bound in Theorem~\ref{WardMethod} is also tight up to logarithmic factors of $k$.

These results and arguments indeed justify our motivation for combining VaDE with Ward's method. To further demonstrate our approach's power, we experimentally evaluate the clustering performance of three methods on high-dimensional equally separated Gaussian mixtures -- the vanilla Ward's method, 
the PCA-Ward combination, and the VaDE-Ward combination. As shown in Figure~\ref{fig:shift-orig}, 
in terms of the dendrogram purity,
Ward's method with VaDE pre-processing tops
the other two, for nearly every mean-separation level. In terms of the final clustering result in the (projected) space, Figures~\ref{fig:spherical-original} to~\ref{fig:spherical-vade} clearly shows the advantage of the VaDE-Ward method as it significantly enlarges the separation among different clusters while making much fewer mistakes in the clustering accuracy.


Next, we show that given the exact recovery in Theorem~\ref{WardMethod} and mild
cluster separation conditions, 
Ward's method recovers the exact underlying hierarchy. 

For any index set $I\subseteq [k]$ and corresponding Gaussian components $\mathcal G_i, i\in I$, 
denote by $w_I$ the total weight $\sum_{i\in I} w_i$, and by $C_I$ the union of sample points from these components. We say that a collection
$\mathcal H$ of nonempty subsets of $[k]$ form a \emph{hierarchy of order $\tau$ over $[k]$} if there exists a index set sequence
$T_\ell$ for $\ell=1,\ldots,\tau$ such that 
i) for every $\ell$, the set $\mathcal H_\ell$ partitions $[k]$;
ii) $\mathcal H_k$ is the union of 
$\mathcal H_\ell:=\{I_{\ell t}: t\in T_\ell\}$;
iii) for every $\ell< \tau$ and $t\in T_\ell$, we have $I_{\ell t}\subseteq I_{(\ell+1)t'}$ for some $t'\in T_{\ell+1}$. 

For any sample size $n$ and Gaussian component $\mathcal G_i$, denote by $S_i:=\sigma_i(\sqrt d+ 2\sqrt{\log n})$, 
which upper bounds the radius of the corresponding sample cluster, with high probability. 
The triangle inequality then implies that for any $i\not= j\in [k]$, the distance between a point in the sample cluster of $\mathcal G_i$ and 
that of $\mathcal G_j$ is \emph{at most} $D_{ij}^+:=\norm{\mu_i-\mu_j}_2+S_i+S_j$, and \emph{at least} $D_{ij}^-:=\norm{\mu_i-\mu_j}_2-S_i-S_j$. 
Naturally, we also define $D_{I,J}^+:=\max_{i\in I, j\in J} D_{ij}^+$ and $D_{I,J}^-:=\min_{i\in I, j\in J} D_{ij}^-$ for any $I, J\subseteq [k]$,
serving as upper and lower bounds on the cluster-level inter-point distances.
These compact notations give our theoretical claim a simple form.

\begin{Theorem}\label{thm:ward-hierarchical}
There exists an absolute constant $c_1<8$ such that the following holds. 
Suppose we have $n$ samples from a $k$-mixture of spherical Gaussians that satisfy the conditions and sample size bound in Theorem~\ref{WardMethod}, 
and suppose there is an 
underlying hierarchy of the Gaussian components $\mathcal H$ satisfying
\[
\forall \ell\in [s], 
\ \ 
I \not= J\in \mathcal H_\ell, 
\ \
D_{I,J}^-
\ge 
c_1 \sqrt{\nu_{\ell}}\max\{D_{I,I}^+ , D_{J,J}^+\},
\]
where $\nu_\ell:=\max_{I \not= J\in \mathcal H_\ell}w_I/w_J$. Then, Ward's method recovers 
the underlying hierarchy $\mathcal H$ with probability at least $1-2/k$.


\end{Theorem}

\begin{figure*}[t]
\begin{center}
   \caption{In the original space, Ward's method makes several mistakes: three orchids merge into the {\em aquatic animals} cluster before the {\em flowers} cluster. Using the VaDE embedding, Ward's method produces a near optimal dendrogram on this subset of data for both lower and higher level structures.}
 \includegraphics[width=.92\textwidth]{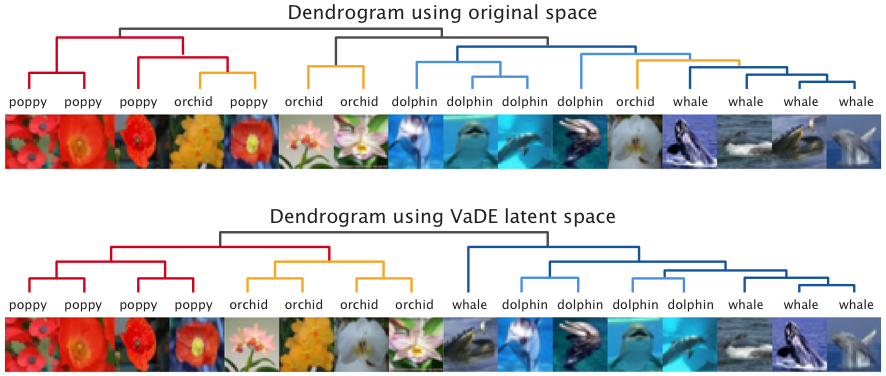}
 \label{fig:cifar100_dendro}
\end{center}  
\end{figure*}
\begin{table*}[ht]\scriptsize
\caption{Dendrogram Purity and Moseley-Wang's objective on BTGM synthetic dataset
with $\alpha=2$.
We evaluate Ward's method using different embeddings, while varying parameters: $k$ denotes the number of pure clusters, $h$ denotes both the dimension of the PCA and VaDE latent spaces and the tree height.}
    \centering
\subfloat[Dendrogram Purity]{
\begin{tabular}{rccc}
\toprule
& $k$ = 8\ ($h$ = 3) & $k$ = 16\ ($h$ = 4)& $k$ = 32\ ($h$ = 5)\\
\midrule
Original & 0.798 $\pm$ .020& 0.789 $\pm$ .017 & 0.775 $\pm$ .016  \\
PCA & 0.518 $\pm$ .002 &0.404 $\pm$ .003 & 0.299 $\pm$ .003 \\
VAE & 0.619 $\pm$ .030 & 0.755 $\pm$ .019  & 0.667 $\pm$ .012 \\
VaDE & \bf{0.943 $\pm$ .011} &\bf{0.945 $\pm$ .001}&\bf{0.869 $\pm$ .013}\\
\bottomrule
\end{tabular}
\hspace{2ex}
}
\subfloat[Moseley-Wang (ratio vs. {\sc opt})]{
\begin{tabular}{ccc}
\toprule
$k = 8\ (h = 3)$ & $k = 16\ (h = 4)$ & $k = 32\ (h = 5)$\\
\midrule 
0.960 $\pm$ 0.004 & 0.980 $\pm$ 0.002 & 0.989 $\pm$ 0.001  \\
 0.909 $\pm$ 0.001 &0.913 $\pm$ 0.004 & 0.936 $\pm$ 0.001 \\
 0.791 $\pm$ 0.028 & 0.898 $\pm$ 0.017 & 0.969 $\pm$ 0.011\\
\bf{0.990 $\pm$ 0.020 }& \bf{0.994 $\pm$ 0.001} & \bf{0.991 $\pm$ 0.001}\\
\bottomrule
\end{tabular}
}
\label{synthetic_table}
\end{table*}

\begin{table*}[th]
\caption{DP on four real datasets for different embedding methods.}
    \centering
\begin{tabular}{r c c c c}
\toprule
& \textsc{Reuters}& \textsc{MNIST} & \textsc{CIFAR-25} & \textsc{20 Newsgroups}\\
\midrule 
Original  & 0.535 & 0.478 & 0.110 & 0.225\\
PCA  & 0.587 & 0.472 & 0.108 & 0.203\\
VAE & 0.468  & 0.495 & 0.082 & 0.196\\
VaDE  &  0.650  & 0.803 & 0.120  & 0.248\\
VaDE + Trans.  & {\bf 0.672} & {\bf 0.886} & \bf{0.128} & {\bf 0.251} \\
\bottomrule
\end{tabular}
\label{comparison_table_DP}
\end{table*}

\begin{table*}[th]
\caption{MW (ratio vs. {\sc opt}) on four real datasets for different embedding methods.}
    \centering
\begin{tabular}{r c c c c}
\toprule
& \textsc{Reuters}& \textsc{MNIST} & \textsc{CIFAR-25} & \textsc{20 Newsgroups}\\
\midrule 
Original &  0.654  & 0.663  & 0.423 & 0.557\\
PCA &  0.679  & 0.650   & 0.435 & 0.575\\
VAE &  0.613  & 0.630  & 0.393 & 0.548\\
VaDE &  0.745  & 0.915  & 0.455 & 0.595\\
VaDE + Trans.  & {\bf 0.768 } & {\bf 0.955 }  & {\bf 0.472} & {\bf 0.605}\\
\bottomrule
\end{tabular}
\label{comparison_table_MW}
\end{table*}


The above theorem resembles Theorem~\ref{WardMethod},
with separation conditions comparing the cluster radiuses with the inter-cluster distances at different hierarchy levels.
Informally, the theorem states that if the clusters in each level of the hierarchy are well-separated, Ward's method will be able to recover the correct clustering of that level, and
hence also the entire hierarchy. 

Recall from Section~\ref{sec:prelims} that a BTGM, Binary Tree Gaussian Mixture, is a mixture of Gaussians that encode the hierarchy. For the BTGM in Eq.~(\ref{eq:btgm}), the shared variance implies that
$S:=S_i=\sqrt d + 2\sqrt{\log n}$ is a constant for any $i\in [k]$, yielding 
$D_{I,I}^+=\max_{i,j\in I} \norm{\mu_i - \mu_j}_2+2S$ for any $I\subseteq [k]$.
Then by symmetry, we see that $\nu_\ell = 1$,\ $\max_{i,j\in I} \norm{\mu_i - \mu_j}_2+2S$, and 
$D_{I,J}^- + 2S 
\ge 
\alpha(D_{I,I}^+ - 2S)\ge \alpha m$ 
for any level~$\ell$ and $I\not=J\in \mathcal H_\ell$. Consolidating Theorem~\ref{WardMethod} and~\ref{thm:ward-hierarchical} with these simplifications yields
the following result. 

\begin{Corollary}\label{cor:btgm}
Consider the BTGM in Equation~\ref{eq:btgm} and let $k:=2^h$ be the component number. Let $c, c_0$, and $c_1$ be the absolute constants in Theorem~\ref{WardMethod} and~\ref{thm:ward-hierarchical}, and define a constant $c_2:=2(2c_1+1)/(\alpha-c_1)$.
If the sample size satisfies $n\ge c_0\cdot k\log k$ and the model is well-separated with $\alpha>c_1$ and
$m\ge 2\max\{c,c_2\}\cdot (\sqrt{d} +\sqrt{\log n})$, then Ward's method recovers the underlying hierarchy $\mathcal H$ with probability at least $1-2/k$. 
\end{Corollary}


\section{Experiments}
\label{sec:experiments}

Our main experimental goal is to determine which embedding method produces the best results for dendrogram purity and Moseley-Wang's objective\footnote{The source code and pretrained models are available in \href{https://github.com/EHC-nips/Unsupervised_embedding_of_hierarchical_structure_in_euclidean_space}{this link}}. In this section, we focus on Ward's method for clustering after embedding (Appendix~\ref{app:more-experiments} has other algorithms). 


\smallpar{Set-up and methods} 
We compare VaDE against PCA (linear baseline) and a standard VAE (non-linear baseline).  We run Ward's method on the embedded or original data to get dendrograms. 
We test on four real datasets: Reuters and MNIST with flat clustering labels (i.e., 1-level hierarchy); 20 Newsgroups (3-level hierarchy) and CIFAR-25 (2-level hierarchy).
We only use 25 classes from CIFAR-100 that fall into one of the five superclasses ``aquatic animals'', ``flowers'', ``fruits and vegetables'', ``large natural outdoor scenes'' and ``vehicle1''. We call it CIFAR-25. For 20 Newsgroups and CIFAR-25, we use the full dataset, but for MNSIT and Reuters, we average results over 100 runs, each with 10,000 random samples. Embedding methods are trained on the whole dataset. Details in Appendix~\ref{app:more-experiments}. We set the margin $m$ of the BTGM  to be 8 while varying the depth, so the closest distance between the means of any two Gaussian distributions is 8.



\smallpar{Synthetic Data (BTGM)}
 We vary the depth and margin of BTGM to evaluation various hierarchy depths and cluster separations. After generating $k$ means for the BTGM, we shift all non-zero entries of the first $\frac{k}{2}$ means to an arbitrary location. We show that this simple non-linear transformation is challenging for linear embedding methods (i.e., PCA), even though it preserves hierarchical structure. To generate $k$ clusters with shifted means, we sample 2000 points from each Gaussian distribution of the BTGM of height $h$ such that $2^h = k$ for $h = 3,4,5$. Then, $h$ is the embedding dimension of PCA, VAE, and VaDE.
 
\smallpar{Dendrogram Purity (DP)} We compute the DP of various methods using Eq.~(\ref{eq:dp}). We construct ground truth clusters either from class labels of real datasets or from the leaf clusters of the BTGM model.
 
\smallpar{Moseley-Wang's Objective (MW)} 
To evaluate the MW objective in Eq.~(\ref{eq:mw}), we need pairwise similarities $w_{ij}$. We define the similarities based on the ground truth labels and hierarchy. Let $h$ denote the depth of the hierarchy (which ranges between 1 and 5 in our experiments). We use 
$w_{ij} = 2^{l-1} \cdot \mathbbm{1}_{\C^l(x^i) = \C^l(x^j)}$, $l \in [h]$, where $\C^l(x^i)$ denotes the label for point $x^i$ at the $l^{th}$ level of the hierarchical tree from top to bottom. We report the {\bf ratio} achieved by various methods to the optimal MW value {\sc opt} that we compute as follows. When there are hierarchical labels, we compare against the optimal tree with ground truth levels. Otherwise, for flat clusters, {\sc opt} is obtained by merging points in the same clusters, then building a binary tree on top.


\subsection{Results}

\smallpar{Results on Synthetic Data}
Table~\ref{synthetic_table} presents the results of the synthetic experiments.
Embedding with VaDE and then clustering with Ward's method leads to the best DP and MW results in all settings. For DP, we see that PCA performs the worst, while the original (non-embedded) data has the second-best results, and the difference between the results is larger than the empirical standard deviation. For MW, the results are more similar to each other, and perhaps surprisingly, the relative performance changes with the depth.  The latent distributions learned by VaDE are already well-separated, so the rescaling is not necessary in this case. 

\smallpar{Results on Real Data} 
Tables \ref{comparison_table_DP} and \ref{comparison_table_MW} report DP and MW objectives on real datasets.  We observe that for MNIST, using the VaDE embedding improves DP and MW significantly compared to other embeddings, and applying the transformation further leads to an increased value of both objectives. The results on the Reuters and 20 Newsgroups datasets show that VaDE plus the rescaling transformation is also suitable for text data, where it leads to the best performance. In the CIFAR-25 experiments, we see only a modest improvement with VaDE when evaluating with all 25 classes.  We also observe from the table that these two objectives are positively correlated for most of the cases. The only exception occurs when we apply PCA to CIFAR-25 and 20 Newsgroups. PCA embedding increases the MW objective but decreases the DP comparing to the original data. Note that we utilize the hierarchical labels while computing the MW objective for CIFAR-25 and 20 Newsgroups. One possible reason is that PCA can only separate super-clusters in the higher level, but it fails to separate clusters in the lower level of the hierarchical tree. Therefore it results in worse DP but better MW performance.

\subsection{Discussion}

In Theorems~\ref{WardMethod} and~\ref{thm:ward-hierarchical}, and Corollary~\ref{cor:btgm}, we studied the performance of Ward's method on the fundamental Gaussian mixture model. We showed that with high probability, Ward's method perfectly recovers the underlying hierarchy under mild and near-optimal separability assumptions in Euclidean distance. Since VaDE uses a GMM prior, this is a natural metric for the embedded space. As the rescaling increases the separation, we suppose that the embedded vectors behave as if they are approximately sampled from a distribution satisfying the conditions of the theorems. Consequently, our theoretical results justify the success of the embedding in practice, where we see that VaDE+Transformation performs the best on many datasets. We hypothesize that VaDE generates a latent space that encodes more of the high level structure for the CIFAR-25 images. This hypothesis can be partially verified by the qualitative evaluation in  Figure~\ref{fig:cifar100_dendro}. A trend consistent with the synthetic experiments is that PCA and VAE are not always helpful embeddings for hierarchical clustering with Ward's method. Overall, VaDE discovers an appropriate embedding into Euclidean space based on its GMM prior that improves the clustering results.

The synthetic results in Table~\ref{synthetic_table} exhibit the success of the VaDE+Transformation method to recover the underlying clusters and the planted hierarchy. This trend holds while varying the number of clusters and depth of the hierarchy. 
Qualitatively, we find that VaDE can ``denoise'' the sampled points when the margin does not suffice to guarantee non-overlapping clusters. This enables Ward's method to achieve a higher value of DP and MW than running the algorithm on the original data. 
In many cases, embedding with PCA or VAE is actually detrimental to the two objectives compared to the original data. We conjecture that the difference in performance arises because VaDE fully utilizes the $h$-dimensional latent space, but PCA fails to resolve the non-linear shifting of the underlying means. In general, using a linear embedding method will likely lead to a drop in both DP and MW for this distribution. Similarly, VAE does not provide sufficient separation between clusters. The BTGM experiments showed (i) the necessity of using non-linear embedding methods  and (ii) VaDE's ability to increase the separation between clusters.

The rescaling transformation leads to consistently better results for both DP and MW on multiple real datasets. In practice, the data distributions are correlated and overlapping, without clear cluster structure in high-dimensional space. 
The VaDE embedding leads to more concentrated clusters, and the rescaling transformation enlarges the separation enough for Ward's method to recover what is learned by VaDE.

\section{Conclusion}
We investigated the effectiveness of embedding data in Euclidean space to improve unsupervised hierarchical clustering methods. Motivated by unsupervised tasks such as taxonomy construction and phylogeny reconstruction, we studied the improvement gained by combining embedding methods with linkage-based clustering algorithms.
We saw that rescaling the VaDE latent space embedding, and then applying Ward's linkage-based method, leads to the best results for dendrogram purity and the Moseley-Wang objective on many datasets. To better understand the VaDE embedding, we proposed a planted hierarchical model using Gaussians with separated means. We proved that Ward's method recovers both the clustering and the hierarchy with high probability in this model. 
Compared to methods that use hyperbolic embeddings~\cite{gu2018learning, mathieu2019continuous, nickel2017poincare} or ultrametrics~\cite{chierchia2019ultrametric}, we believe that our approach will be easier to integrate into existing systems.

While we focus on embedding using variational autoencoders, an open direction for could involve embedding hierarchical structure using other representation learning methods~\cite{agarwal2007generalized, borg2005modern, caron2018deep, chierchia2019ultrametric, mishne2019diffusion, nina2019decoder, salakhutdinov2012one, tsai2017learning, yadav2019supervised}. 
Another direction is to better understand the similarities and differences between learned embeddings, comparison-based methods, and ordinal relations~\cite{emamjomeh2018adaptive, ghoshdastidar2019foundations,jamieson2011low, kazemi2018comparison}. Finally, a large number of clusters may be a limitation of the VaDE embedding, and this deserves attention.
\bibliographystyle{plain}
\setcitestyle{numbers}
\bibliography{main}

\begin{thebibliography}{10}

\bibitem{abboud2019subquadratic}
Amir Abboud, Vincent Cohen-Addad, and Hussein Houdroug{\'e}.
\newblock Subquadratic high-dimensional hierarchical clustering.
\newblock In {\em Advances in Neural Information Processing Systems}, pages
  11576--11586, 2019.

\bibitem{ackermann2014analysis}
Marcel~R Ackermann, Johannes Bl{\"o}mer, Daniel Kuntze, and Christian Sohler.
\newblock Analysis of agglomerative clustering.
\newblock {\em Algorithmica}, 69(1):184--215, 2014.

\bibitem{agarwal2007generalized}
Sameer Agarwal, Josh Wills, Lawrence Cayton, Gert Lanckriet, David Kriegman,
  and Serge Belongie.
\newblock Generalized non-metric multidimensional scaling.
\newblock In {\em Artificial Intelligence and Statistics}, pages 11--18, 2007.

\bibitem{ahmadian2019bisect}
Sara Ahmadian, Vaggos Chatziafratis, Alessandro Epasto, Euiwoong Lee, Mohammad
  Mahdian, Konstantin Makarychev, and Grigory Yaroslavtsev.
\newblock Bisect and conquer: Hierarchical clustering via max-uncut bisection.
\newblock {\em arXiv preprint arXiv:1912.06983}, 2019.

\bibitem{aljalbout2018clustering}
Elie Aljalbout, Vladimir Golkov, Yawar Siddiqui, Maximilian Strobel, and Daniel
  Cremers.
\newblock Clustering with deep learning: Taxonomy and new methods.
\newblock {\em arXiv preprint arXiv:1801.07648}, 2018.

\bibitem{balaban2019treecluster}
Metin Balaban, Niema Moshiri, Uyen Mai, Xingfan Jia, and Siavash Mirarab.
\newblock Treecluster: Clustering biological sequences using phylogenetic
  trees.
\newblock {\em PloS one}, 14(8), 2019.

\bibitem{balakrishnan2011noise}
Sivaraman Balakrishnan, Min Xu, Akshay Krishnamurthy, and Aarti Singh.
\newblock Noise thresholds for spectral clustering.
\newblock In {\em Advances in Neural Information Processing Systems}, pages
  954--962, 2011.

\bibitem{balcan2019learning}
Maria-Florina Balcan, Travis Dick, and Manuel Lang.
\newblock Learning to link.
\newblock {\em arXiv preprint arXiv:1907.00533}, 2019.

\bibitem{blum2020foundations}
Avrim Blum, John Hopcroft, and Ravindran Kannan.
\newblock {\em Foundations of data science}.
\newblock Cambridge University Press, 2020.

\bibitem{borg2005modern}
Ingwer Borg and Patrick~JF Groenen.
\newblock {\em Modern multidimensional scaling: Theory and applications}.
\newblock Springer Science \& Business Media, 2005.

\bibitem{caron2018deep}
Mathilde Caron, Piotr Bojanowski, Armand Joulin, and Matthijs Douze.
\newblock Deep clustering for unsupervised learning of visual features.
\newblock In {\em Proceedings of the European Conference on Computer Vision
  (ECCV)}, pages 132--149, 2018.

\bibitem{chami2020trees}
Ines Chami, Albert Gu, Vaggos Chatziafratis, and Christopher R{\'e}.
\newblock From trees to continuous embeddings and back: Hyperbolic hierarchical
  clustering.
\newblock {\em arXiv preprint arXiv:2010.00402}, 2020.

\bibitem{charikar2017approximate}
Moses Charikar and Vaggos Chatziafratis.
\newblock Approximate hierarchical clustering via sparsest cut and spreading
  metrics.
\newblock In {\em Proceedings of the Twenty-Eighth Annual ACM-SIAM Symposium on
  Discrete Algorithms}, pages 841--854. SIAM, 2017.

\bibitem{charikar2018hierarchical}
Moses Charikar, Vaggos Chatziafratis, Rad Niazadeh, and Grigory Yaroslavtsev.
\newblock Hierarchical clustering for euclidean data.
\newblock {\em arXiv preprint arXiv:1812.10582}, 2018.

\bibitem{chierchia2019ultrametric}
Giovanni Chierchia and Benjamin Perret.
\newblock Ultrametric fitting by gradient descent.
\newblock In {\em Advances in neural information processing systems}, pages
  3175--3186, 2019.

\bibitem{chung2006complex}
Fan Chung and Linyuan Lu.
\newblock {\em Complex graphs and networks}, volume 107.
\newblock American Mathematical Soc., 2006.

\bibitem{Cohen-Addad}
Vincent Cohen-Addad, Varun Kanade, and Frederik Mallmann-Trenn.
\newblock Hierarchical clustering beyond the worst-case.
\newblock In {\em Advances in Neural Information Processing Systems}, pages
  6201--6209, 2017.

\bibitem{dasgupta2002performance}
Sanjoy Dasgupta.
\newblock Performance guarantees for hierarchical clustering.
\newblock In {\em International Conference on Computational Learning Theory},
  pages 351--363. Springer, 2002.

\bibitem{dasgupta2016cost}
Sanjoy Dasgupta.
\newblock A cost function for similarity-based hierarchical clustering.
\newblock In {\em Proceedings of the forty-eighth annual ACM symposium on
  Theory of Computing}, pages 118--127, 2016.

\bibitem{de2018representation}
Christopher De~Sa, Albert Gu, Christopher R{\'e}, and Frederic Sala.
\newblock Representation tradeoffs for hyperbolic embeddings.
\newblock {\em Proceedings of machine learning research}, 80:4460, 2018.

\bibitem{dilokthanakul2016deep}
Nat Dilokthanakul, Pedro~AM Mediano, Marta Garnelo, Matthew~CH Lee, Hugh
  Salimbeni, Kai Arulkumaran, and Murray Shanahan.
\newblock Deep unsupervised clustering with gaussian mixture variational
  autoencoders.
\newblock {\em arXiv preprint arXiv:1611.02648}, 2016.

\bibitem{emamjomeh2018adaptive}
Ehsan Emamjomeh-Zadeh and David Kempe.
\newblock Adaptive hierarchical clustering using ordinal queries.
\newblock In {\em Proceedings of the Twenty-Ninth Annual ACM-SIAM Symposium on
  Discrete Algorithms}, pages 415--429. SIAM, 2018.

\bibitem{fard2018deep}
Maziar~Moradi Fard, Thibaut Thonet, and Eric Gaussier.
\newblock Deep $ k $-means: Jointly clustering with $ k $-means and learning
  representations.
\newblock {\em arXiv preprint arXiv:1806.10069}, 2018.

\bibitem{feller1968introduction}
William Feller.
\newblock {\em An introduction to probability theory and its applications. Vol.
  1}.
\newblock John Wiley \& Sons,, 1968.

\bibitem{figueroa2017simple}
Jhosimar~Arias Figueroa and Ad{\'\i}n~Ram{\'\i}rez Rivera.
\newblock Is simple better?: Revisiting simple generative models for
  unsupervised clustering.
\newblock In {\em Second workshop on Bayesian Deep Learning (NIPS)}, 2017.

\bibitem{friedman2001elements}
Jerome Friedman, Trevor Hastie, and Robert Tibshirani.
\newblock {\em The elements of statistical learning}, volume 1:10.
\newblock Springer Series in Statistics, New York, 2001.

\bibitem{ganea2018hyperbolic}
Octavian Ganea, Gary Becigneul, and Thomas Hofmann.
\newblock Hyperbolic entailment cones for learning hierarchical embeddings.
\newblock In {\em International Conference on Machine Learning}, pages
  1646--1655, 2018.

\bibitem{ghoshdastidar2019foundations}
Debarghya Ghoshdastidar, Micha{\"e}l Perrot, and Ulrike von Luxburg.
\newblock Foundations of comparison-based hierarchical clustering.
\newblock In {\em Advances in Neural Information Processing Systems}, pages
  7454--7464, 2019.

\bibitem{goldberger2005hierarchical}
Jacob Goldberger and Sam~T Roweis.
\newblock Hierarchical clustering of a mixture model.
\newblock In {\em Advances in Neural Information Processing Systems}, pages
  505--512, 2005.

\bibitem{goyal2017nonparametric}
Prasoon Goyal, Zhiting Hu, Xiaodan Liang, Chenyu Wang, and Eric~P Xing.
\newblock Nonparametric variational auto-encoders for hierarchical
  representation learning.
\newblock In {\em Proceedings of the IEEE International Conference on Computer
  Vision}, pages 5094--5102, 2017.

\bibitem{grosswendt2019analysis}
Anna Gro{\ss}wendt, Heiko R{\"o}glin, and Melanie Schmidt.
\newblock Analysis of ward's method.
\newblock In {\em Proceedings of the Thirtieth Annual ACM-SIAM Symposium on
  Discrete Algorithms}, pages 2939--2957. SIAM, 2019.

\bibitem{grosswendt_2020}
Anna-Klara Gro{\ss}wendt.
\newblock {\em {Theoretical Analysis of Hierarchical Clustering and the Shadow
  Vertex Algorithm}}.
\newblock PhD thesis, Universit\"{a}t Bonn, 2020.

\bibitem{gu2018learning}
Albert Gu, Frederic Sala, Beliz Gunel, and Christopher Ré.
\newblock Learning mixed-curvature representations in product spaces.
\newblock In {\em International Conference on Learning Representations}, 2019.

\bibitem{guo2017deep}
Xifeng Guo, Xinwang Liu, En~Zhu, and Jianping Yin.
\newblock Deep clustering with convolutional autoencoders.
\newblock In {\em International conference on neural information processing},
  pages 373--382. Springer, 2017.

\bibitem{heller2005bayesian}
Katherine~A Heller and Zoubin Ghahramani.
\newblock Bayesian hierarchical clustering.
\newblock In {\em Proceedings of the 22nd international conference on Machine
  learning}, pages 297--304, 2005.

\bibitem{Gaussian2016notes}
Daniel Hsu.
\newblock {Columbia COMS 4772 Fall 2016, Lecture Notes: High-dimensional
  Gaussians}, 2016.
\newblock URL:
  \url{https://www.cs.columbia.edu/~djhsu/coms4772-f16/lectures/gaussians.md.handout.pdf}.
  Last visited on 2020/05/22.

\bibitem{jamieson2011low}
Kevin~G Jamieson and Robert~D Nowak.
\newblock Low-dimensional embedding using adaptively selected ordinal data.
\newblock In {\em 2011 49th Annual Allerton Conference on Communication,
  Control, and Computing (Allerton)}, pages 1077--1084. IEEE, 2011.

\bibitem{jiang2017variational}
Zhuxi Jiang, Yin Zheng, Huachun Tan, Bangsheng Tang, and Hanning Zhou.
\newblock Variational deep embedding: an unsupervised and generative approach
  to clustering.
\newblock In {\em Proceedings of the 26th International Joint Conference on
  Artificial Intelligence}, pages 1965--1972, 2017.

\bibitem{kazemi2018comparison}
Ehsan Kazemi, Lin Chen, Sanjoy Dasgupta, and Amin Karbasi.
\newblock Comparison based learning from weak oracles.
\newblock In {\em International Conference on Artificial Intelligence and
  Statistics}, pages 1849--1858, 2018.

\bibitem{king1967step}
Benjamin King.
\newblock Step-wise clustering procedures.
\newblock {\em Journal of the American Statistical Association},
  62(317):86--101, 1967.

\bibitem{kingma2014auto}
Diederik~P Kingma and Max Welling.
\newblock Auto-encoding variational bayes.
\newblock {\em stat}, 1050:1, 2014.

\bibitem{klushyn2019learning}
Alexej Klushyn, Nutan Chen, Richard Kurle, Botond Cseke, and Patrick van~der
  Smagt.
\newblock Learning hierarchical priors in vaes.
\newblock In {\em Advances in Neural Information Processing Systems}, pages
  2866--2875, 2019.

\bibitem{kobren2017hierarchical}
Ari Kobren, Nicholas Monath, Akshay Krishnamurthy, and Andrew McCallum.
\newblock A hierarchical algorithm for extreme clustering.
\newblock In {\em Proceedings of the 23rd ACM SIGKDD International Conference
  on Knowledge Discovery and Data Mining}, pages 255--264, 2017.

\bibitem{kpotufepruning}
Samory Kpotufe and Ulrike von Luxburg.
\newblock Pruning nearest neighbor cluster trees.
\newblock In {\em International Conference on Machine Learning (ICML)}, 2011.

\bibitem{CIFAR100}
Alex Krizhevsky.
\newblock Learning multiple layers of features from tiny images.
\newblock Technical report, University of Toronto, 2009.

\bibitem{lance1967general}
Godfrey~N Lance and William~Thomas Williams.
\newblock A general theory of classificatory sorting strategies: 1.
  hierarchical systems.
\newblock {\em The computer journal}, 9(4):373--380, 1967.

\bibitem{Lang95newsgroup}
Ken Lang.
\newblock Newsweeder: Learning to filter netnews.
\newblock In {\em Proceedings of the Twelfth International Conference on
  Machine Learning}, pages 331--339, 1995.

\bibitem{MNIST}
Y.~{Lecun}, L.~{Bottou}, Y.~{Bengio}, and P.~{Haffner}.
\newblock Gradient-based learning applied to document recognition.
\newblock {\em Proceedings of the IEEE}, 86(11):2278--2324, 1998.

\bibitem{li2018discriminatively}
Fengfu Li, Hong Qiao, and Bo~Zhang.
\newblock Discriminatively boosted image clustering with fully convolutional
  auto-encoders.
\newblock {\em Pattern Recognition}, 83:161--173, 2018.

\bibitem{li2018learning}
Xiaopeng Li, Zhourong Chen, Leonard~KM Poon, and Nevin~L Zhang.
\newblock Learning latent superstructures in variational autoencoders for deep
  multidimensional clustering.
\newblock In {\em International Conference on Learning Representations}, 2018.

\bibitem{lin2010general}
Guolong Lin, Chandrashekhar Nagarajan, Rajmohan Rajaraman, and David~P
  Williamson.
\newblock A general approach for incremental approximation and hierarchical
  clustering.
\newblock {\em SIAM Journal on Computing}, 39(8):3633--3669, 2010.

\bibitem{linial1995geometry}
Nathan Linial, Eran London, and Yuri Rabinovich.
\newblock The geometry of graphs and some of its algorithmic applications.
\newblock {\em Combinatorica}, 15(2):215--245, 1995.

\bibitem{Amini09RCV1}
C.~Goutte M.-R.~Amini, N.~Usunier.
\newblock Learning from multiple partially observed views - an application to
  multilingual text categorization.
\newblock In {\em Advances in Neural Information Processing Systems}, pages
  28--36, 2009.

\bibitem{manning2008introduction}
Christopher~D Manning, Prabhakar Raghavan, and Hinrich Sch{\"u}tze.
\newblock {\em Introduction to information retrieval}.
\newblock Cambridge university press, 2008.

\bibitem{mathieu2019continuous}
Emile Mathieu, Charline Le~Lan, Chris~J Maddison, Ryota Tomioka, and Yee~Whye
  Teh.
\newblock Continuous hierarchical representations with poincar{\'e} variational
  auto-encoders.
\newblock In {\em Advances in neural information processing systems}, pages
  12544--12555, 2019.

\bibitem{min2018survey}
Erxue Min, Xifeng Guo, Qiang Liu, Gen Zhang, Jianjing Cui, and Jun Long.
\newblock A survey of clustering with deep learning: From the perspective of
  network architecture.
\newblock {\em IEEE Access}, 6:39501--39514, 2018.

\bibitem{mishne2019diffusion}
Gal Mishne, Uri Shaham, Alexander Cloninger, and Israel Cohen.
\newblock Diffusion nets.
\newblock {\em Applied and Computational Harmonic Analysis}, 47(2):259--285,
  2019.

\bibitem{monath2019gradient}
Nicholas Monath, Manzil Zaheer, Daniel Silva, Andrew McCallum, and Amr Ahmed.
\newblock Gradient-based hierarchical clustering using continuous
  representations of trees in hyperbolic space.
\newblock In {\em 25th ACM SIGKDD International Conference on Knowledge
  Discovery \& Data Mining}, pages 714--722, 2019.

\bibitem{MoseleyWang}
Benjamin Moseley and Joshua Wang.
\newblock Approximation bounds for hierarchical clustering: Average linkage,
  bisecting k-means, and local search.
\newblock In {\em Advances in Neural Information Processing Systems}, pages
  3094--3103, 2017.

\bibitem{nalisnick2016approximate}
Eric Nalisnick, Lars Hertel, and Padhraic Smyth.
\newblock {Approximate Inference for Deep Latent Gaussian Mixtures}.
\newblock In {\em NIPS Workshop on Bayesian Deep Learning}, 2016.

\bibitem{nalisnick2017stick}
Eric Nalisnick and Padhraic Smyth.
\newblock Stick-breaking variational autoencoders.
\newblock In {\em International Conference on Learning Representations (ICLR)},
  2017.

\bibitem{nickel2017poincare}
Maximillian Nickel and Douwe Kiela.
\newblock Poincar{\'e} embeddings for learning hierarchical representations.
\newblock In {\em Advances in neural information processing systems (NIPS)},
  pages 6338--6347, 2017.

\bibitem{nina2019decoder}
Oliver Nina, Jamison Moody, and Clarissa Milligan.
\newblock A decoder-free approach for unsupervised clustering and manifold
  learning with random triplet mining.
\newblock In {\em Proceedings of the IEEE International Conference on Computer
  Vision Workshops}, pages 0--0, 2019.

\bibitem{reynolds2006clustering}
Alan~P Reynolds, Graeme Richards, Beatriz de~la Iglesia, and Victor~J
  Rayward-Smith.
\newblock Clustering rules: a comparison of partitioning and hierarchical
  clustering algorithms.
\newblock {\em Journal of Mathematical Modelling and Algorithms},
  5(4):475--504, 2006.

\bibitem{rezende2014stochastic}
Danilo~Jimenez Rezende, Shakir Mohamed, and Daan Wierstra.
\newblock Stochastic backpropagation and approximate inference in deep
  generative models.
\newblock In {\em International Conference on Machine Learning (ICML)}, pages
  1278--1286, 2014.

\bibitem{roux2018comparative}
Maurice Roux.
\newblock A comparative study of divisive and agglomerative hierarchical
  clustering algorithms.
\newblock {\em Journal of Classification}, 35(2):345--366, 2018.

\bibitem{salakhutdinov2012one}
Ruslan Salakhutdinov, Joshua Tenenbaum, and Antonio Torralba.
\newblock One-shot learning with a hierarchical nonparametric bayesian model.
\newblock In {\em Proceedings of ICML Workshop on Unsupervised and Transfer
  Learning}, pages 195--206, 2012.

\bibitem{sarkar2011low}
Rik Sarkar.
\newblock Low distortion delaunay embedding of trees in hyperbolic plane.
\newblock In {\em International Symposium on Graph Drawing}, pages 355--366.
  Springer, 2011.

\bibitem{shang2020nettaxo}
Jingbo Shang, Xinyang Zhang, Liyuan Liu, Sha Li, and Jiawei Han.
\newblock Nettaxo: Automated topic taxonomy construction from text-rich
  network.
\newblock In {\em Proceedings of The Web Conference 2020}, pages 1908--1919,
  2020.

\bibitem{sharma2019comparative}
Shweta Sharma, Neha Batra, et~al.
\newblock Comparative study of single linkage, complete linkage, and ward
  method of agglomerative clustering.
\newblock In {\em 2019 International Conference on Machine Learning, Big Data,
  Cloud and Parallel Computing (COMITCon)}, pages 568--573. IEEE, 2019.

\bibitem{shin2019hierarchically}
Su-Jin Shin, Kyungwoo Song, and Il-Chul Moon.
\newblock Hierarchically clustered representation learning.
\newblock {\em arXiv preprint arXiv:1901.09906}, 2019.

\bibitem{sneath1973numerical}
Peter~HA Sneath and Robert~R Sokal.
\newblock {\em Numerical taxonomy. The principles and practice of numerical
  classification.}
\newblock W.H. Freeman, 1973.

\bibitem{sonthalia2020tree}
Rishi Sonthalia and Anna~C. Gilbert.
\newblock Tree! i am no tree! i am a low dimensional hyperbolic embedding,
  2020.

\bibitem{szekely2005hierarchical}
Gabor~J Szekely and Maria~L Rizzo.
\newblock Hierarchical clustering via joint between-within distances: Extending
  ward's minimum variance method.
\newblock {\em Journal of classification}, 22(2), 2005.

\bibitem{teh2008bayesian}
Yee~W Teh, Hal Daume~III, and Daniel~M Roy.
\newblock Bayesian agglomerative clustering with coalescents.
\newblock In {\em Advances in Neural Information Processing Systems}, pages
  1473--1480, 2008.

\bibitem{tifrea2018poincar}
Alexandru Tifrea, Gary B{\'e}cigneul, and Octavian-Eugen Ganea.
\newblock {Poincare GloVe: Hyperbolic Word Embeddings}.
\newblock {\em arXiv preprint arXiv:1810.06546}, 2018.

\bibitem{tokuda2020revisiting}
Eric~K. Tokuda, Cesar~H. Comin, and Luciano da~F.~Costa.
\newblock Revisiting agglomerative clustering.
\newblock {\em arXiv preprint arXiv:2005.07995}, 2020.

\bibitem{tomczak2018vae}
Jakub Tomczak and Max Welling.
\newblock Vae with a vampprior.
\newblock In {\em International Conference on Artificial Intelligence and
  Statistics}, pages 1214--1223, 2018.

\bibitem{tsai2017learning}
Yao-Hung~Hubert Tsai, Liang-Kang Huang, and Ruslan Salakhutdinov.
\newblock Learning robust visual-semantic embeddings.
\newblock In {\em 2017 IEEE International Conference on Computer Vision
  (ICCV)}, pages 3591--3600. IEEE, 2017.

\bibitem{tzoreff2018deep}
Elad Tzoreff, Olga Kogan, and Yoni Choukroun.
\newblock Deep discriminative latent space for clustering.
\newblock {\em arXiv preprint arXiv:1805.10795}, 2018.

\bibitem{uugur2020variational}
Yi{\u{g}}it U{\u{g}}ur and Abdellatif Zaidi.
\newblock Variational information bottleneck for unsupervised clustering: Deep
  gaussian mixture embedding.
\newblock {\em Entropy}, 22(2):213, 2020.

\bibitem{vasconcelos1999learning}
Nuno Vasconcelos and Andrew Lippman.
\newblock Learning mixture hierarchies.
\newblock In {\em Advances in Neural Information Processing Systems}, pages
  606--612, 1999.

\bibitem{ward1963hierarchical}
Joe~H Ward~Jr.
\newblock Hierarchical grouping to optimize an objective function.
\newblock {\em Journal of the American statistical association},
  58(301):236--244, 1963.

\bibitem{xie2016unsupervised}
Junyuan Xie, Ross Girshick, and Ali Farhadi.
\newblock Unsupervised deep embedding for clustering analysis.
\newblock In {\em International conference on machine learning}, pages
  478--487, 2016.

\bibitem{yadav2019supervised}
Nishant Yadav, Ari Kobren, Nicholas Monath, and Andrew McCallum.
\newblock Supervised hierarchical clustering with exponential linkage.
\newblock In {\em International Conference on Machine Learning (ICML)}, 2019.

\bibitem{yang2017towards}
Bo~Yang, Xiao Fu, Nicholas~D Sidiropoulos, and Mingyi Hong.
\newblock Towards k-means-friendly spaces: Simultaneous deep learning and
  clustering.
\newblock In {\em Proceedings of the 34th International Conference on Machine
  Learning-Volume 70}, pages 3861--3870. JMLR. org, 2017.

\bibitem{zhang2018taxogen}
Chao Zhang, Fangbo Tao, Xiusi Chen, Jiaming Shen, Meng Jiang, Brian Sadler,
  Michelle Vanni, and Jiawei Han.
\newblock Taxogen: Unsupervised topic taxonomy construction by adaptive term
  embedding and clustering.
\newblock In {\em Proceedings of the 24th ACM SIGKDD International Conference
  on Knowledge Discovery \& Data Mining}, pages 2701--2709, 2018.

\bibitem{BIRCH1996}
Tian Zhang, Raghu Ramakrishnan, and Miron Livny.
\newblock Birch: An efficient data clustering method for very large databases.
\newblock In {\em Proceedings of the 1996 ACM SIGMOD International Conference
  on Management of Data}, SIGMOD ’96, page 103–114, New York, NY, USA,
  1996. Association for Computing Machinery.

\end{thebibliography}

\appendix
\section{Proof of the Theorems}
\label{theorem_proof}
The next two lemmas provide tight tail bounds for binomial distributions and spherical Gaussians.
 \begin{Lemma}[Concentration of binomial random variables]\label{con_bin}\cite{chung2006complex}
Let $X$ be a binomial random variable with mean $\mu$, then for any $\delta>0$, 
\[
\Pr(X\geq{(1+\delta)\mu}) 
\leq{e^{-(\max\{\delta^2,\delta\}) \mu/3}},
\]
and for any $\delta\in{(0, 1)}$, 
\[
\Pr(X\leq{(1-\delta)\mu}) 
\leq{e^{-\delta^2\mu/2}}.
\]
\end{Lemma}

\begin{Lemma}[Concentration of spherical Gaussians]\label{con_gaussian}\cite{Gaussian2016notes}
For a random 
$X$ distributed according to a spherical Gaussian in $d$ dimensions (mean $\mu$ and variance $\sigma^2$ in each direction),
and any $\delta\in (0,1)$,
\[
\Pr\left(
\norm{X-\mu}_2^2 \le \sigma^2 d \left(
1 + 2\sqrt{\frac{\log(1/\delta)}{d}} + \frac{2\log(1/\delta)}{d}
\right)\right)
\ge 1-\delta.
\]
\end{Lemma}

\begin{Lemma}[Optimal clustering through Ward's method]\label{opt_ward}\cite{grosswendt2019analysis, grosswendt_2020}
Let $P\subseteq \mathbb R^d$ be a collection of points with underlying clustering $C_1^*,\ldots, C_k^*$, 
such that the corresponding cluster mean values $\mu_1^*,\ldots, \mu_k^*$ satisfy 
$\norm{\mu_i^*-\mu_j^*}_2> (2+2\sqrt{2\nu})\max_{x\in C_i^*}\norm{x-\mu_i^*}_2$ for all $i\not=j$,
where $\nu := \max_{\ell,t\in [k]}\frac{|C_\ell^*|}{|C_t^*|}$ is the maximum ratio between cluster sizes. 
Then Ward's method recovers the underlying clustering. 
\end{Lemma}

Now we are ready to prove Theorem 1~\ref{WardMethod}. 
\begin{proof}[\bf{Proof of Theorem~\ref{WardMethod}}]
 First, for each component in the mixture, we determine the associated number of 
    sample points. This is equivalent to drawing a size-$n$ sample from the weighting distribution 
    that assigns $w_i>0$ to each $i\in [k]$. The number of times each symbol $i\in [k]$ appearing in the sample,
    which we refer to as the (sample) \emph{multiplicity} of the symbol,
    follows a binomial distribution with mean $nw_i$. 
    By Lemma~\ref{con_bin}, the probability that an arbitrary multiplicity 
    is within a factor of $2$ from its mean value is at least $1-2e^{-nw_i/8}$.
    Hence, for $n\ge 16(\log k)/w_i$, the union bound yields that this deviation bound
    holds for all symbols with probability at least $1-k\cdot k^{-2} = 1-k^{-1}$.  
    
 Second, given the symbol multiplicities, we draw sample points from each component 
    correspondingly. Consider the $i$-th component with mean $\mu_i$. Then, by Lemma~\ref{con_gaussian}, 
    every sample point from this component deviates (under Euclidean distance) from $\mu_i$ by at most 
    $S_i:=\sigma_i(\sqrt{d}+2\sqrt{\log n})$, with probability at least $1-n^{-2}$. 
    Again, the union bound yields that this deviation bound holds for all components and all sample points
    with probability at least $1-n\cdot n^{-2} = 1 - n^{-1}$.

 Third, by the above reasoning and $n>k$, with probability at least $1-2/k$ and for all $i\in [k]$,
 \begin{enumerate}
     \item the sample size associated with the $i$-th component is between $nw_i/2$ and $2nw_i$; 
     \item every sample point from the $i$-th component deviates from $\mu_i$ by at most $S_i$. 
 \end{enumerate}

    Therefore, if further $\norm{\mu_i-\mu_j}_2\ge (5+12\max_{\ell,t\in [k]} \sqrt{w_\ell/w_t})(S_i+S_j)$ for all $i\not=j$,
    the sample points will form well-separated clusters according to the underlying components.
    In addition, the first assumption ensures that the maximum ratio between cluster sizes is 
    at most $4\max_{\ell,t\in [k]} \sqrt{w_\ell/w_t}$, and the second assumption implies that for any $i\not= j\in [k]$,
    \begin{enumerate}
    \item the $i$-th cluster
    the distance between the mean to any point in it is at most $2S_i$;
    \item the empirical means of the $i$-th and $j$-th clusters are separated by at least 
    \[
    16(S_i+S_j)>\left(2+2\sqrt{2\cdot 4}\cdot \max_{\ell,t\in [k]} \sqrt{\frac{w_\ell}{w_t}}\right) \max\{2S_i, 2S_j\}. 
    \]
    \end{enumerate}

Finally, applying Lemma~\ref{opt_ward} completes the proof of the theorem.
\end{proof}

Next, we move to the proof of Theorem~\ref{thm:ward-hierarchical}.
Let us first recap what the theorem states.
\setcounter{Theorem}{1}
\begin{Theorem}\label{thm:ward-hierarchical-1}
There exists an absolute constant $c_1>0$ such that the following holds. 
Suppose we are given a sample of size $n$ from a $k$-mixture of spherical Gaussians that satisfy the separation conditions and the sample size lower bound in Theorem~\ref{WardMethod}, 
and suppose that there is an 
underlying hierarchy of the Gaussian components $\mathcal H$ satisfying
\[
\forall \ell\in [s], \quad I \not= J\in \mathcal H_\ell, \quad 
D_{I,J}^-\ge c_1 \sqrt{\nu_{\ell}}\ (D_{I,I}^+ + D_{J,J}^+),
\]
where $\nu_\ell:=\max_{I \not= J\in \mathcal H_\ell}w_I/w_J$. Then, Ward's method recovers 
the underlying hierarchy $\mathcal H$ with probability at least $1-2/k$.


\end{Theorem}

\begin{proof}[\bf{Proof of Theorem~\ref{thm:ward-hierarchical}}]
By the proof of Theorem 1~\ref{WardMethod},
for any sample size $n$ and Gaussian component $\mathcal G_i$, the quantity $S_i:=\sigma_i(\sqrt d+ 2\sqrt{\log n})$ upper bounds the radius of the corresponding sample cluster, with high probability. Here, the term ``radius'' is defined with respect to the \emph{actual center} $\mu_i$.

The triangle inequality then implies that for any $i\neq j\in [k]$, the distance between a point in the sample cluster of $\mathcal G_i$ and 
that of $\mathcal G_j$ is \emph{at most} $D_{ij}^+:=\norm{\mu_i-\mu_j}_2+S_i+S_j$, i.e., the sum of the distance between the two centers and the radiuses of the two clusters, 
and \emph{at least} $D_{ij}^-:=\norm{\mu_i-\mu_j}_2-S_i-S_j$. 

More generally, for any non-singleton clusters represented by $I, J\subset [k]$, the distance between 
two points $x\in I$ and $y\in J$ is at most
$D_{I,J}^+:=\max_{i\in I, j\in J} D_{ij}^+$, and at least $D_{I,J}^-:=\min_{i\in I, j\in J} D_{ij}^-$. 
In particular, for $I=J$, quantity $D_{I,I}^+$ upper bounds the diameter of cluster $I$. 

Furthermore, by the proof of Theorem~\ref{WardMethod}, given the exact recovery of the $k$ Gaussian components, 
the ratio between the sizes of any two clusters (at any level), say $C_I$ and $C_J$, differs from 
the ratio between their weights, $w_I/w_J$, by a factor of at most $4$. Hence, the quantity $v_\ell$ essentially characterizes the maximum ratio between the sample cluster sizes at level $\ell$ of the hierarchy. 

With the above reasoning, Lemma~\ref{opt_ward} naturally completes the proof of the theorem. 
Intuitively, this shows that if the clusters in each level of the hierarchy are well-separated, Ward's method will be able to recover the correct clustering of that level, and
hence also the entire hierarchy. 
\end{proof}


\section{Optimally of Theorem~\ref{WardMethod}}
\label{theorem_optimality}
Below, we illustrate the hardness of improving the theorem without preprocessing the data, e.g., 
applying an embedding method to the sample points, such as VaDE or PCA.

\paragraph{Separation conditions.}
As many other commonly used linkage methods, the behavior and performance of Ward's method  
depend on only the inter-point distances between the observations. 
For optimal-clustering recovery through Ward's method, 
a natural assumption to make is that the distances between points in the same cluster is 
at most that between disjoint clusters. 
Consider a simple example where we draw $3$ sample points from a uniform 
mixture of two standard $d$-dimensional spherical 
Gaussians $\mathcal G_0$ and $\mathcal G_1$ with mean separated by a distance $\Delta$. 
Then, with a $3/8$ probability, we will draw two sample points from $\mathcal G_0$, say $X, X_0$,
and one sample point from $\mathcal G_1$, say $Y$.
Conditioned on this, the reasoning in Section 2.8 of~\cite{blum2020foundations} implies that, with high probability, 
\[
D_{\mathcal G_0}:=\norm{X-X_0}_2^2\approx 2d\pm \mathcal{O}(\sqrt d) 
\text{ and } 
D_{\mathcal G_1}:=\norm{X-Y}_2^2\approx \Delta^2+2d\pm \mathcal{O}(\sqrt d).
\]
Hence, the assumption that points are closer inside the clusters translates to 
\[
D_{\mathcal G_0}<D_{\mathcal G_1}
\implies 2d\pm \mathcal{O}(\sqrt d)< \Delta^2+2d\pm \mathcal{O}(\sqrt d)
\implies \Delta= \Omega(d^{1/4}).
\]
Furthermore, if one requires the inner-cluster distance to be 
slightly smaller than the inter-cluster distance with high probability, 
say $(1+\epsilon)D_{\mathcal G_0}<D_{\mathcal G_1}$, where $\epsilon>0$ 
is a small constant, 
\[
(1+\epsilon)(2d\pm \mathcal{O}(\sqrt d))< \Delta^2+2d\pm \mathcal{O}(\sqrt d)
\implies
\Delta = \Omega( \sqrt{\epsilon d}).
\]
Therefore, the mean difference between the two Gaussians should exhibit a linear dependence on $\sqrt d$. 

In addition, for a standard $d$-dimensional spherical Gaussian, 
Theorem 2.9 in~\cite{blum2020foundations} states that for any $\beta\le \sqrt d$,
all but at most $3\exp(-\Omega(\beta^2))$ of its probability mass  
lies within $\sqrt d-\beta\le |x|\le \sqrt d +\beta$. 
Hence, if we increase the sample size from $3$ to $n$, correctly separating the points 
by their associated components requires an extra $\mathcal{O}(\sqrt{\log n})$ term in the mean difference. 

\paragraph{Sample-size lower bound.} In terms of the sample size $n$,
Theorem~\ref{WardMethod} requires $n\ge c_0 \log k /w_{\text{\tiny min}}$. 
For Ward's method to recover the underlying $k$-component clustering with high probability,
we must observe at least one sample point from every component of the Gaussian mixture.
Clearly, this means that the expected size of the smallest cluster should be at least $1$,
implying $n\ge 1/w_{\text{\tiny min}}$. 
If we have $\nu = \mathcal{O}(1)$,
this reduces to a weighted version of the coupon collector problem,
which calls for a sample size in the order of 
$k\log k = \Theta((\log k)/w_{\text{\tiny min}})$ by the standard results~\cite{feller1968introduction}.

\section{More Experimental Details}
\label{app:more-experiments}

\smallpar{Network Architecture}  The architecture for the encoder and decoder for both VaDE and VAE are fully-connected layers with size $d$-2000-500-500-$h$ and $h$-500-500-2000-$d$, where $h$ is the hidden dimension for the latent space $\R^h$, and $d$ is the input dimension of the data in $\R^d$. These settings follow the original paper of VaDE \cite{jiang2017variational}. We use Adam optimizer on mini-batched data of size 800 with learning rate 5e-4 on all three real datasets. The choices of the hyper-parameters such as $h$ and $k$ will be discussed in the Sensitivity Analysis.



\smallpar{Computing Dendrogram Purity}
For every set of $N$ points sampled from the dataset (e.g. we use $N = 2048$ for MNIST, CIFAR-25), we compute the exact dendrogram purity using the formula \ref{eq:dp}. We then report the mean and standard variance of the results by repeating the experiments 100 times.


\smallpar{Sensitivity Analysis}  We check the sensitivity of VAE and VaDE to its hyperparameters $h$ and $k$ in terms of DP and MW. We also check the sensitivity of rescaling factor $s$ applied after the VaDE embedding. The details are presented in Figure \ref{fig:sens:VaDE}, \ref{fig:sens:VAE}, and~\ref{fig:sens:k}. 


 \begin{figure*}
 \begin{center}
 \begin{subfigure}[t]{0.45\textwidth}
 \centering
 \caption{Sensitivity of VaDE to the hidden dimension $h$ while we fix $k=10$}
 \label{fig:sensitivity-h}
 \includegraphics[width=\textwidth]{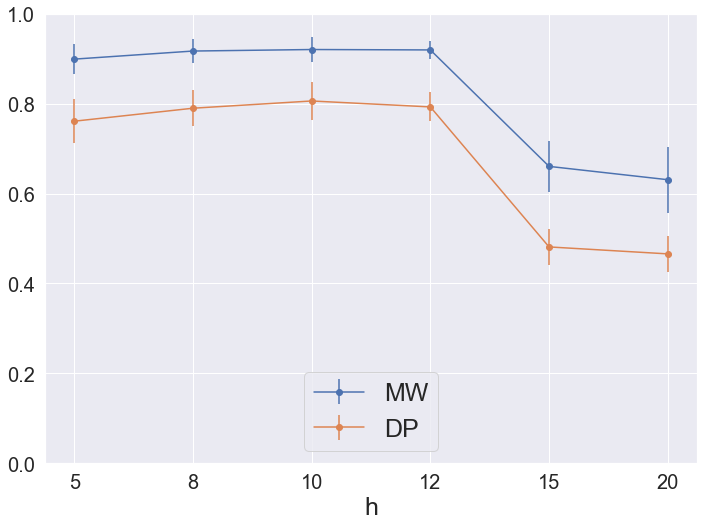}
 \end{subfigure}\hspace{2em}
  \begin{subfigure}[t]{0.45\textwidth}
 \caption{Sensitivity of VaDE to the number of latent clusters $k$ while we fix $h=10$}
 \label{fig:sensitivity-k}
 \includegraphics[width=\textwidth]{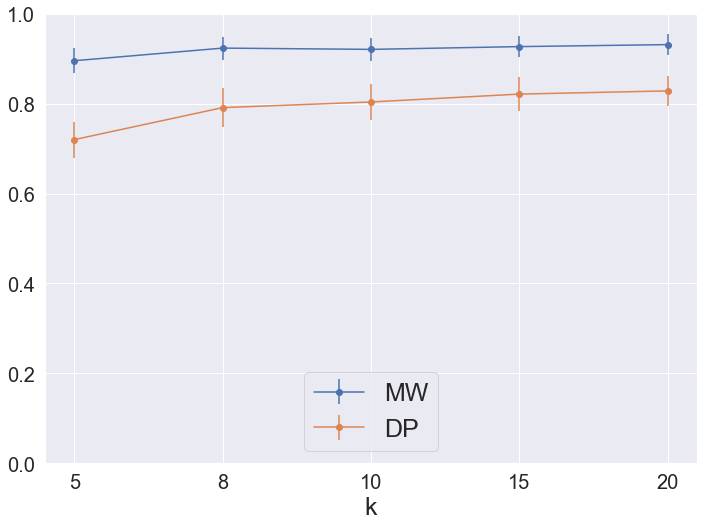}
 \end{subfigure}%
 \end{center}
   \caption{Evaluating the sensitivity of VaDE on MNIST to its hyper-parameters $h$, which denotes the dimensionality of the latent space embedding, and $k$, which denotes the number of GMM components used for the latent space. We plot the Moseley-Wang objective (MW) and dendrogram purity (DP) as we vary either $h$ or $k$. Overall, we see that as long as $h$ is between 5 and 12, then the results are stable. Similarly, $k$ just needs to be at least 8. Therefore, while these parameters are important, the VaDE model is not very sensitive. Note that in our experiments for real data we fixed $h=10$ and for synthetic data we set $h$ to be the dimensionality of the BTGM (either $3,4,$ or $5$). For real/synthetic data we set $k$ to be the number of ground truth clusters (i.e., $k=10$ for MNIST).}
 \label{fig:sens:VaDE}
 \end{figure*}

 \begin{figure*}
 \begin{center}
 \begin{subfigure}[t]{0.45\textwidth}
 \centering
 \caption{Sensitivity of VAE to the hidden dimension $h$ with the same architecture as VaDE}
 \label{fig:VAE-sensitivity-h}
 \includegraphics[width=\textwidth]{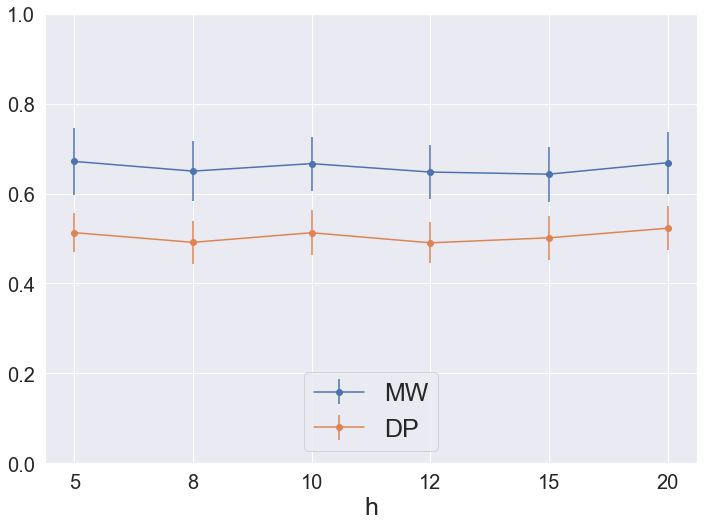}
 \end{subfigure}\hspace{2em}
  \begin{subfigure}[t]{0.45\textwidth}
 \caption{Sensitivity of VaDE  to the rescaling factor $s$ while we fix $h = 10$ and $k=10$}
 \label{fig:sensitivity-s}
 \includegraphics[width=\textwidth]{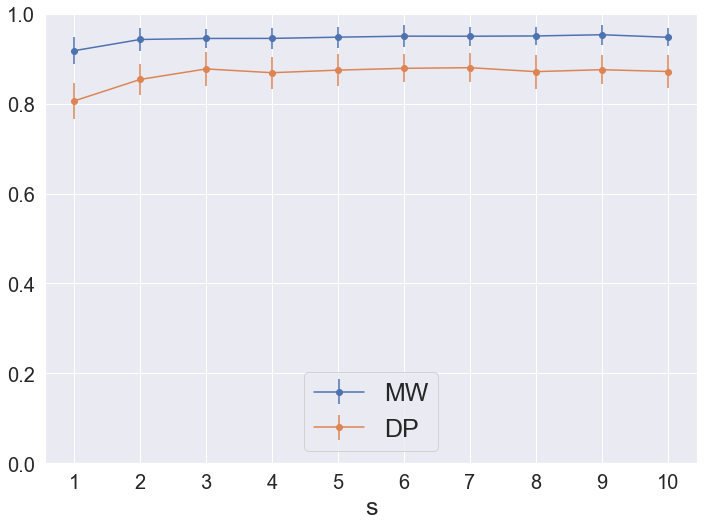}
 \end{subfigure}%
 \end{center}
   \caption{On the left, we evaluate the sensitivity of VAE on MNIST to its hyper-parameters $h$, which denotes the dimensionality of the latent space embedding. On the right, we evaluate the sensitivity of VaDE + Trans to its hyper-parametere $s$. We plot the Moseley-Wang objective (MW) and dendrogram purity (DP) as we vary either $h$ or $s$. We see that the performance of VAE is not sensitive to its hidden dimension size $h$. But overall, the VAE does not perform well with respect to DP and MW. As for the rescaling experiment, we see that the effect of $s$ becomes stable when $s \geq 3$. Note that in the real dataset experiments we set $s = 3$.}
 \label{fig:sens:VAE}
 \end{figure*}

\smallpar{Comparing Linkage-Based Methods} Table~\ref{table:vary_method} shows the DP and MW results for several linkage-based variations, depending on the function used to determine cluster similarity. Overall, we see that Ward's method performs the best on average. However, in some cases, the other methods do comparably.

\begin{table}
\centering
\caption{Dendrogram Purity and Moseley-Wang's objective using different HAC linkage methods on VaDE + Transformed latent space.}
\label{table:vary_method}
\centering
\begin{tabular}{l c c | c c}
\hline
\toprule
& Dendro. Purity & M-W objective & Dendro. Purity & M-W objective\\
\midrule
Single & 0.829 $\pm$ 0.038 & 0.932 $\pm$ 0.023 & 0.260 $\pm$ 0.005 & 0.419 $\pm$ 0.006\\
Complete & 0.854 $\pm$ 0.040 & 0.937 $\pm$ 0.024 & 0.293 $\pm$ 0.008 & 0.451 $\pm$ 0.016\\
Centroid & 0.859 $\pm$ 0.035 & \bf{0.949 $\pm$ 0.019} & 0.284 $\pm$ 0.008 & 0.453 $\pm$ 0.009\\
Average & 0.866 $\pm$ 0.035 & 0.948 $\pm$ 0.023 & 0.295 $\pm$ 0.005 & 0.462 $\pm$ 0.006\\
Ward & \bf{0.870 $\pm$ 0.031} & 0.948 $\pm$ 0.025 & \bf{0.300 $\pm$ 0.008} & \bf{0.465 $\pm$ 0.011}\\
\bottomrule
\end{tabular}
\begin{tabular}{c c c c c c c c c c c c c c}
& & & &MNIST &  & & & & & & & & CIFAR-25-s
\end{tabular}
\end{table}

\begin{figure*}[t!]
    \centering
    \begin{subfigure}[t]{0.3\textwidth}
        \centering
        \includegraphics[height=1in]{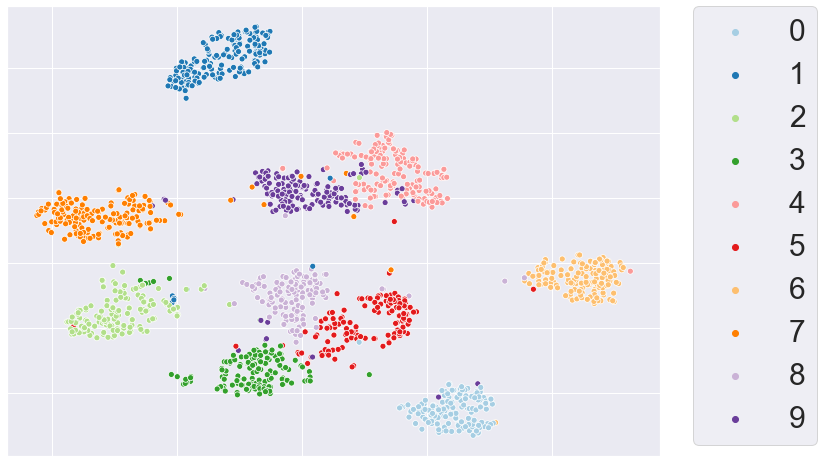}
        \caption{$k = 10$}
        \label{fig:shift-orig2}
    \end{subfigure}%
    ~ 
    \begin{subfigure}[t]{0.3\textwidth}
        \centering
        \includegraphics[height=1in]{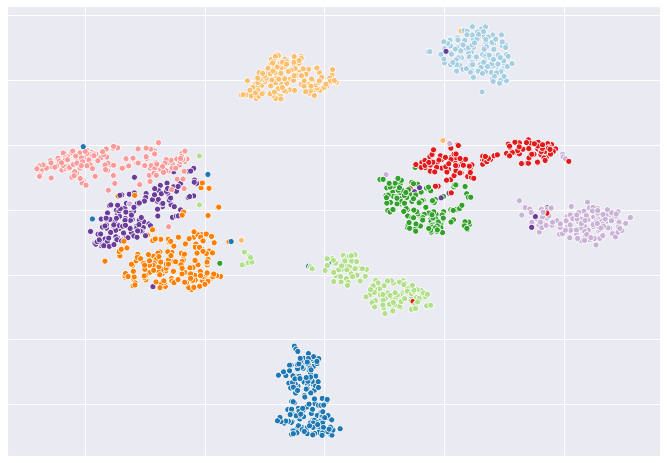}
        \caption{$k = 15$}
        \label{fig:shift-pca2}
    \end{subfigure}
    ~
    \begin{subfigure}[t]{0.3\textwidth}
        \centering
        \includegraphics[height=1in]{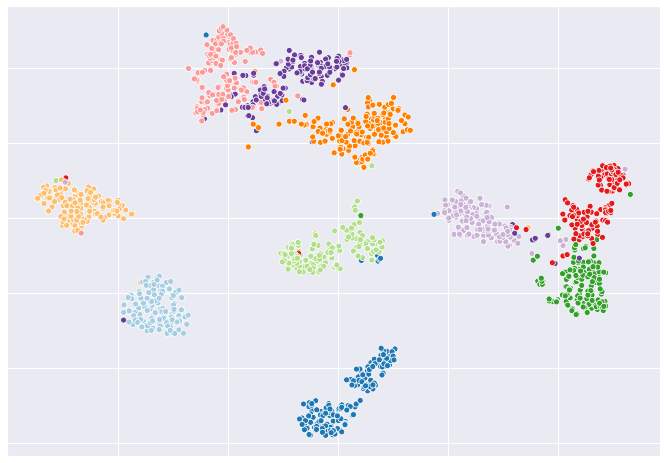}
        \caption{$k = 20$}
        \label{fig:shift-vade2}
    \end{subfigure}
\caption{tSNE visualization of the VaDE latent space with different numbers of latent clusters $k$. We can see that the latent space of VaDE becomes more separated as the number of clusters $k$ increases. When $k = 20$, many of the original 10 clusters of digit further divide into 2 smaller clusters, which in total, form the 20 clusters learned by VaDE.}
\label{fig:sens:k}
\end{figure*}

\smallpar{Dataset Details} We conduct the experiments on our synthetic data from BTGM as well as real data benchmarks: Reuters~\cite{Amini09RCV1}, MNIST~\cite{MNIST} and CIFAR-100~\cite{CIFAR100}, 20 Newsgroup~\cite{Lang95newsgroup}. In the experiments, we only use 25 classes from CIFAR-100 that fall into one of the five superclasses "aquatic animals", "flowers", "fruits and vegetables", "large natural outdoor scenes" and "vehicle1" \footnote{The information of these superclasses can be found in \url{https://www.cs.toronto.edu/~kriz/cifar.html}}, we call it CIFAR-25. Below we also provide further details on the  \textbf{Digits} dataset of handwritten digits \footnote{\url{https://archive.ics.uci.edu/ml/datasets/Optical+Recognition+of+Handwritten+Digits}}. More details about the number of clusters and other statistics can be found in Table~\ref{tab:datasets}. 

\smallpar{Rescaling explained} The motivation of rescaling transformation can be explained using
Figure~\ref{fig:rescaling_explained}. (a) shows the tSNE visualization of the VaDE embedding before and after the rescaling transformation. As we can see, each cluster becomes more separated. In (b), we visualize the inter-class distance matrices before and after the transformation. We use blue and black bounding boxes to highlight few superclusters revealed in the VaDE embedding. These high level structures remain unchanged after the rescaling, which is a consequence of property (ii) in Section~\ref{sec:approach}. 

\smallpar{Separability assumption} It's true that such a property does not hold for the real datasets {\bf before the embedding}. But the assumptions of our theorems are not about the real data themselves, but about the transformed data after the VaDE embedding. We observe that VaDE increases the separation of data from different clusters. Indeed, we verify this for MNIST in Figure \ref{fig:separation} below. We provide further justification for using VaDE via two new theorems that relate separability and Ward's methods.

\begin{table}[th]
\caption{Dataset Details.}
    \centering
\begin{tabular}{c c c c c c}
\toprule
& \textsc{Digits}& \textsc{Reuters}& \textsc{MNIST} & \textsc{20newsgroup} & \textsc{CIFAR-25}\\
\midrule 
\# ground truth clusters & 10 & 4 & 10 & 20 & 25 \\
\# data for training & 6560 & 680K & 60K & 18K & 15K \\
\# data for evaluation & 1000 & 10k & 10k & 18k & 15k \\
Dimensionality & 64 & 2000 & 768 & 2000 & 3072\\
Size of hidden dimension & 5 & 10 & 10 & 20  & 20\\
\bottomrule
\end{tabular}
\label{tab:datasets}
\end{table}

\subsection{Digits Dataset}

There are few experimental results that evaluate dendrogram purity for the \textbf{Digits} dataset in  previous papers.  First we experiment using the same settings and same test set construction for \textbf{Digits} as in  \cite{monath2019gradient}. In this case, VaDE+Ward achieves a DP of 0.941, substantially improving the previous state-of-the-art results for dendrogram purity from existing methods: the DP of~\cite{monath2019gradient} is 0.675, the DP of PERCH~\cite{kobren2017hierarchical}
is 0.614, and the DP of BIRCH \cite{BIRCH1996} is 0.544.
In other words, our approach of VaDE+Ward leads to 26.5 point increase in DP for this simple dataset.

Surprisingly, we find that the test data they used is quite easy, as the DP and MW numbers are substantially better compared to using a random subset as the test set. Therefore, we follow the same evaluation settings as in Section~\ref{sec:experiments} and randomly sampled 100 samples of 1000 data points for evaluation. The results are shown in Table \ref{table:Digits}. The contrast between the previous test data (easier) versus random sampling (harder) further supports our experimental set-up in the main paper.

\begin{table}[ht]
\centering
\caption{Dendrogram Purity and Moseley-Wang's objective using different embedding methods in \textbf{Digits} dataset. Overall, VaDE + Trans. only provides a minor improvement in DP. Using Ward on the original space already achieves decent results.}
\centering
\begin{tabular}{l c c }
\hline
\toprule
& Dendro. Purity & M-W objective \\
\midrule
Original & 0.820 $\pm$ 0.031 & \bf{0.846 $\pm$ 0.029} \\
PCA & 0.823 $\pm$ 0.028 & 0.844 $\pm$ 0.027\\
VaDE & 0.824 $\pm$ 0.026 & 0.840 $\pm$ 0.024 \\
VaDE + Trans. & \bf{0.832 $\pm$ 0.029} & 0.845 $\pm$ 0.024 \\
\bottomrule
\end{tabular}
\label{table:Digits}
\end{table}

%
%


\begin{figure*}[t!] 
    \centering
    \begin{subfigure}[t]{0.45\textwidth}
        \centering
        \includegraphics[height=1.3in]{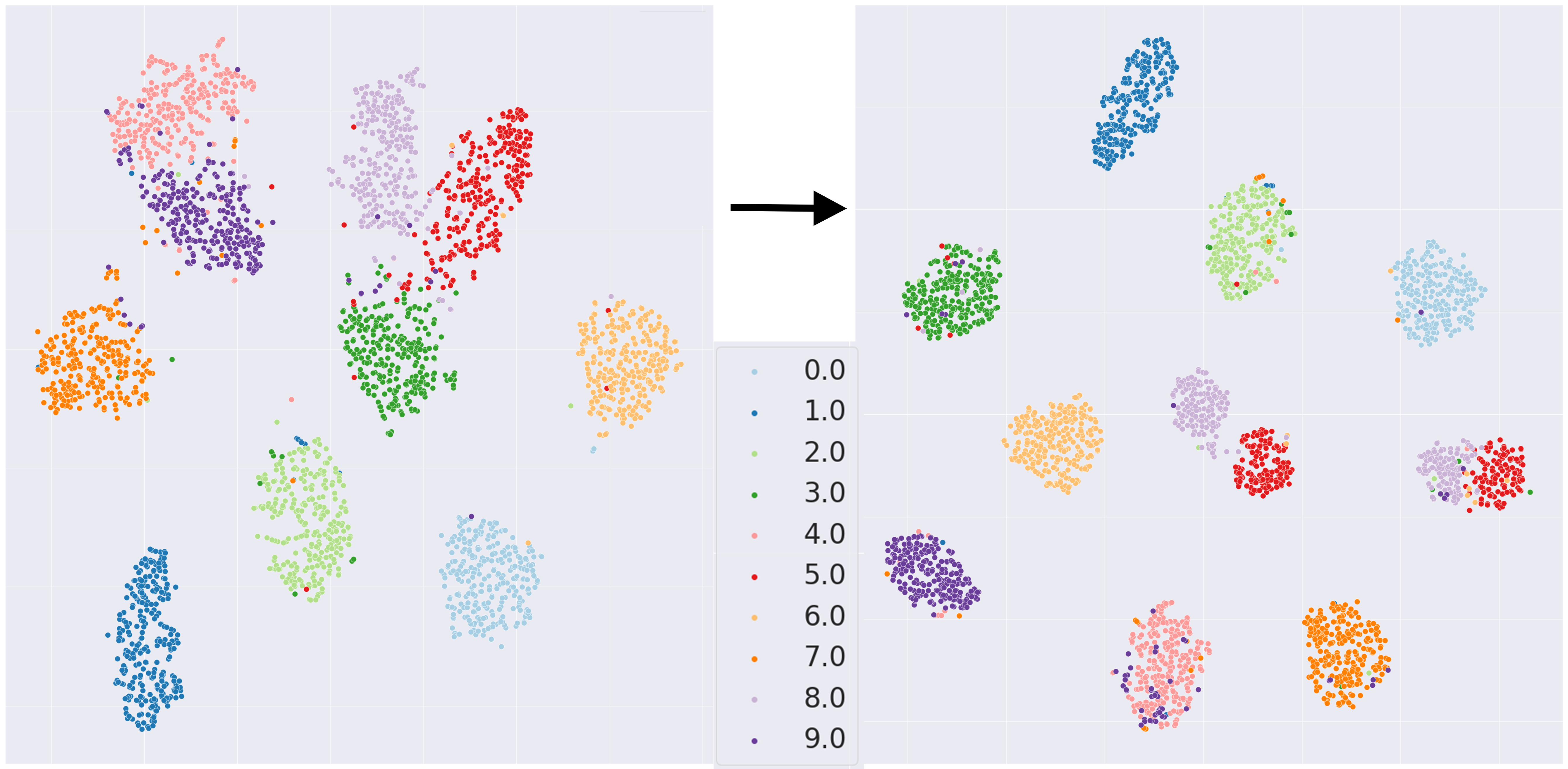}
        \caption{}
        \label{fig:mnist-rescaling-tSNE}
    \end{subfigure}
    ~
    \begin{subfigure}[t]{0.45\textwidth}
        \centering
        \includegraphics[height=1.3in]{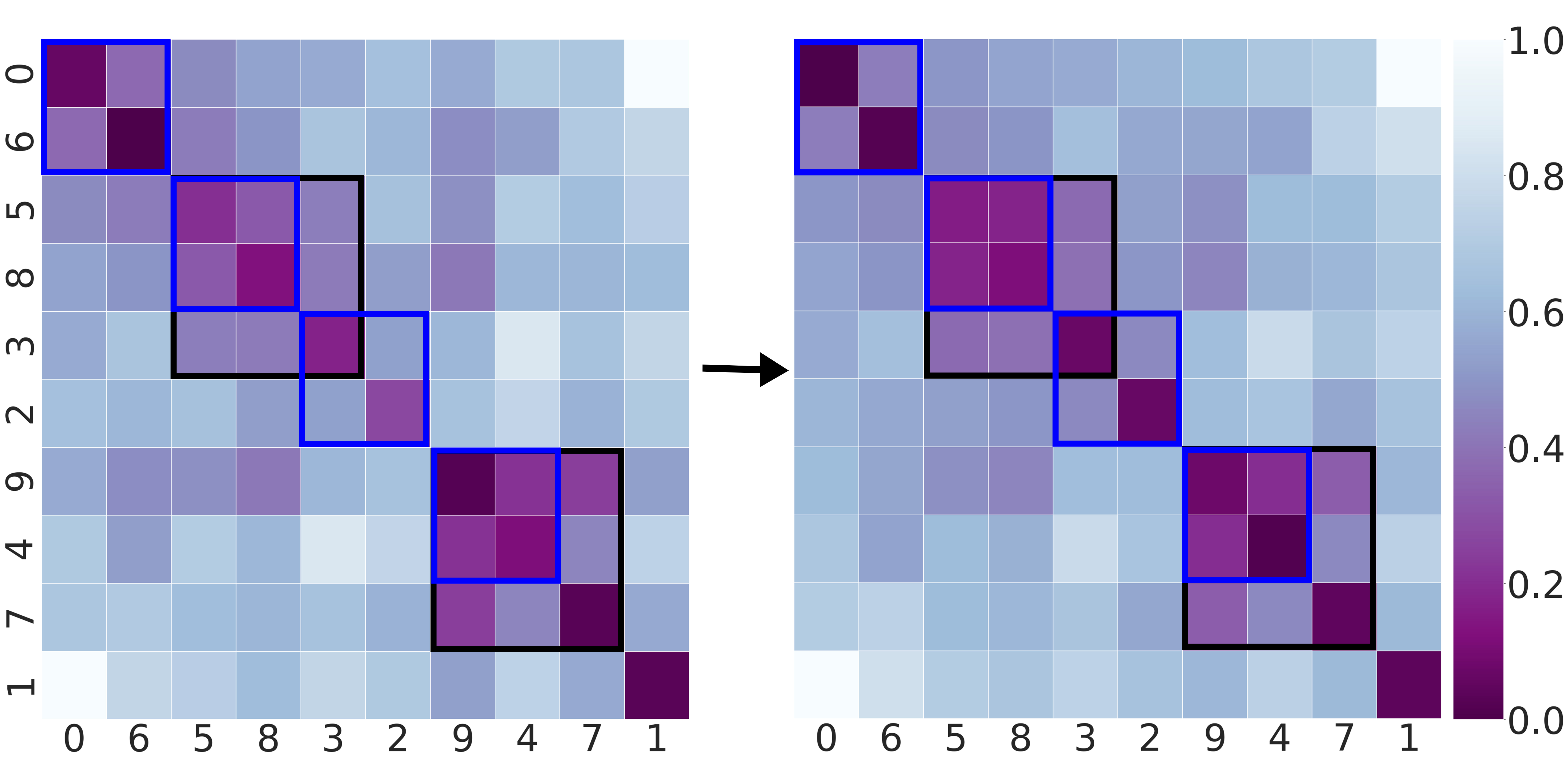}
        \caption{}
        \label{fig:mnist-rescaling-matrix}
    \end{subfigure}
\caption{(a) tSNE visualization of the embedded data for MNIST before and after the rescaling transformation, (b) shows the inter-class distance matrix before and after the rescaling transformation}
\label{fig:rescaling_explained}
\end{figure*}

\begin{figure*}[!h]
    \centering
    \begin{subfigure}[t]{0.48\textwidth}
        \centering
        \includegraphics[height=1.2in]{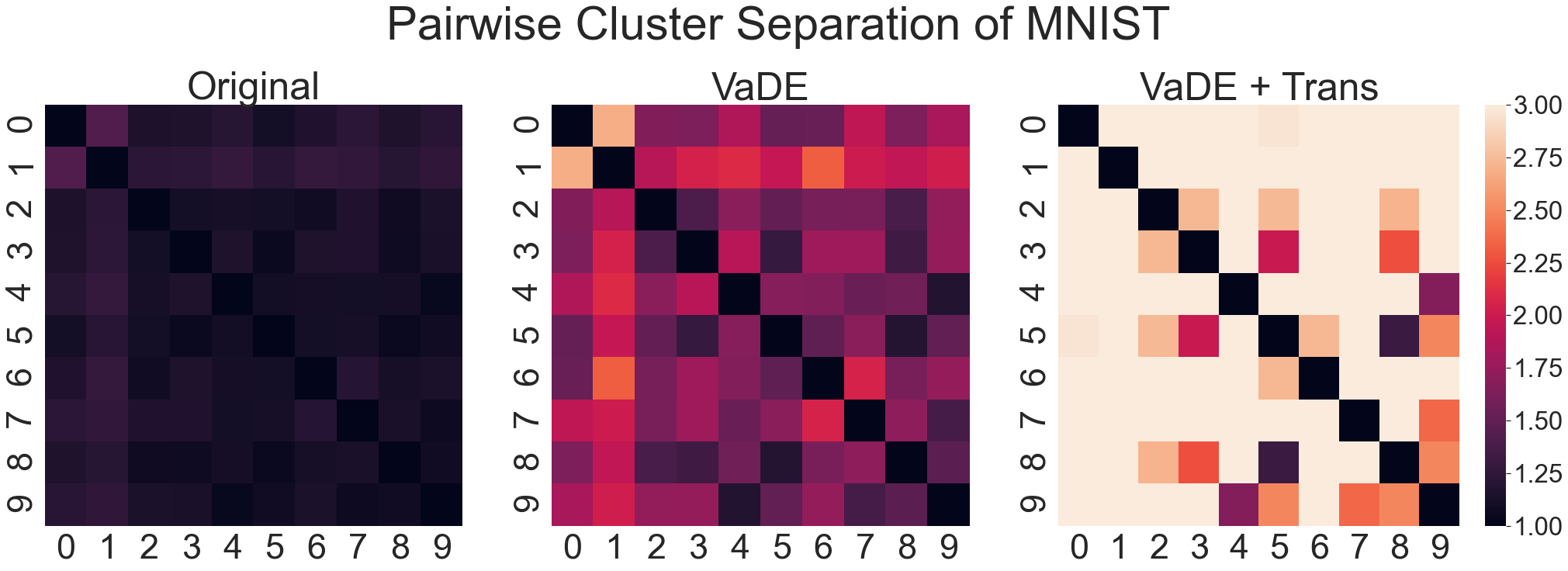}
        \label{fig:MNIST-separation}
    \end{subfigure}
    ~
    \begin{subfigure}[t]{0.48\textwidth}
        \centering
        \includegraphics[height=1.2in]{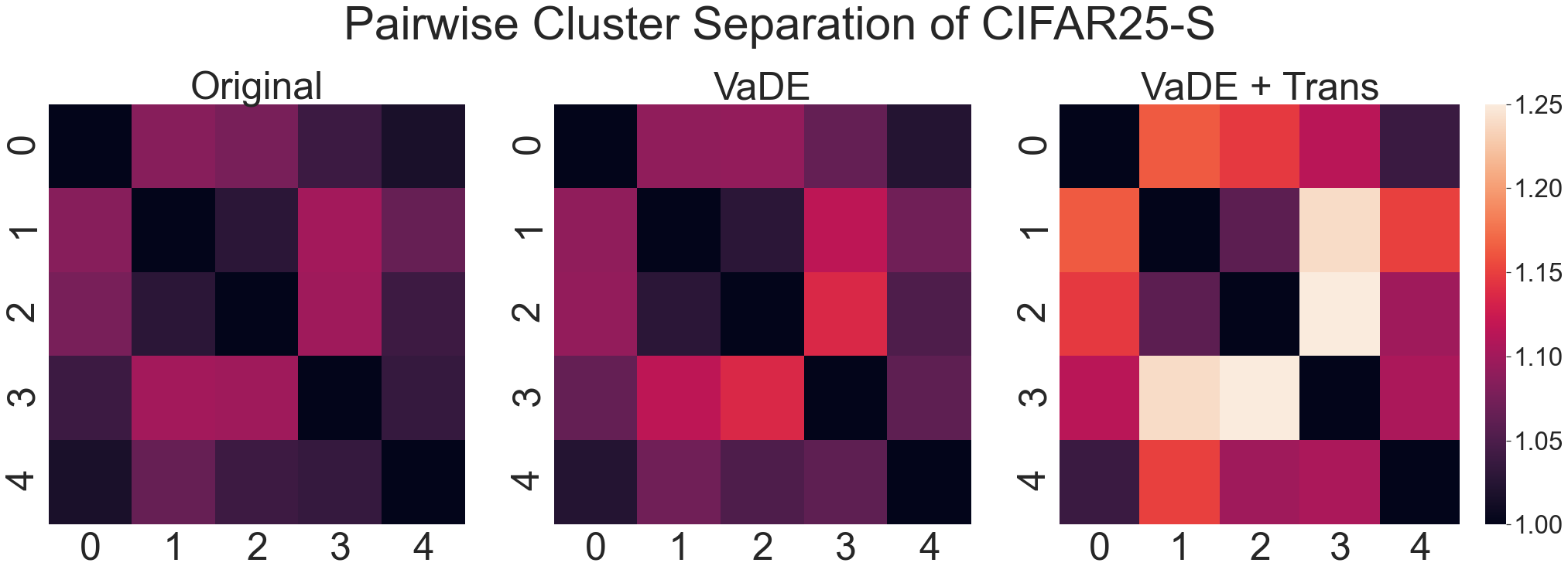}
        \label{fig:CIFAR-separation}
    \end{subfigure}
\caption{ Cluster separation is measured by \tiny $\displaystyle \sum_{x \in \mathcal{C}_i}\sum_{y \in \mathcal{C}_j} \norm{x - y} \left(\sum_{x,y \in \mathcal{C}_i} \norm{x-y}\right)^{-1/2} \left(\sum_{x,y \in \mathcal{C}_j} \norm{x-y} \right)^{-1/2}$.}
\label{fig:separation}
\end{figure*}
~\vfill

\end{document}